\documentclass[journal]{IEEEtran}
\usepackage{xspace}
\usepackage{enumitem}
\usepackage{threeparttable}
\usepackage{amsmath,amssymb,amsthm}
\usepackage{algorithm}
\usepackage{caption}
\usepackage{algpseudocode}
\usepackage{epsfig}
\usepackage{graphicx}
\usepackage{amsmath}
\usepackage{amssymb}
\usepackage{bm}
\usepackage{multirow}
\usepackage{subfigure}
\usepackage{verbatim}
\usepackage{soul, xcolor}
\usepackage{color}
\usepackage[sort,compress]{cite}

\newtheorem{theorem}{Theorem}

\begin{document}

\title{A unified framework for manifold landmarking}

\author{Hongteng Xu,~
        Licheng Yu,~
        Mark A. Davenport,~\IEEEmembership{Senior Member,~IEEE},~
        Hongyuan Zha
\thanks{H.\ Xu and M.\ A.\ Davenport are with the Department
of Electrical and Computer Engineering, Georgia Institute of Technology, Atlanta,
GA, 30332 USA (e-mail: $\{$hxu42,~mdav$\}$@gatech.edu).}
\thanks{L.\ Yu is with the Department of Computer Science, University of
North Carolina at Chapel Hill, Chapel Hill, NC, 27599 USA (e-mail:
licheng@cs.unc.edu).}
\thanks{H.\ Zha is with the College of Computing, Georgia Institute of Technology, Atlanta, GA, 30332 USA (e-mail: zha@cc.gatech.edu).}}

\markboth{IEEE Transactions on Signal Processing,~Vol.~XX, No.~X, XX~201X}%
{Xu \MakeLowercase{\textit{et al.}}: Active manifold learning}

\maketitle

\begin{abstract}
The success of semi-supervised manifold learning is highly dependent on the quality of the labeled samples.
Active manifold learning aims to select and label representative landmarks on a manifold from a given set of samples to improve semi-supervised manifold learning.
In this paper, we propose a novel active manifold learning method based on a unified framework of manifold landmarking.
In particular, our method combines geometric manifold landmarking methods with algebraic ones. We achieve this by using the Gershgorin circle theorem to construct an upper bound  on the learning error that depends on the landmarks and the manifold's alignment matrix in a way that captures both the geometric and algebraic criteria. We then attempt to select landmarks so as to minimize this bound by iteratively deleting the Gershgorin circles corresponding to the selected landmarks.
We also analyze the complexity, scalability, and robustness of our method through simulations, and demonstrate its superiority compared to existing methods.
Experiments in regression and classification further verify that our method performs better than its competitors.
\end{abstract}


\IEEEpeerreviewmaketitle
\section{Introduction}
\IEEEPARstart{S}{emi-supervised} manifold learning methods~\cite{belkin2006manifold,yang2006semi,zhang2008spectral,zhu2009introduction,zheng2011graph,BelkiN_Semi} have been widely used to capture low-dimensional structure in high-dimensional data.
These methods take semantic information (labels) into consideration when learning the mapping from the ambient space to the latent space.
The learned latent variables can be used as features for many learning tasks~\cite{belkin2006manifold,zheng2011graph,xu2015dictionary}.
In some cases~\cite{yang2006semi,zhang2008spectral}, we are even able to learn the mapping from the ambient space to the label space directly and estimate the labels for the complete data set.

An interesting and important problem in the context of semi-supervised manifold learning is \emph{how to select landmarks from a large number of unlabeled samples to minimize the learning error for the remaining samples.}
This problem is very common in practical situations --- given a large number of unlabeled samples, we can often label only a few of them because of limitations in budget, time, and other resources.
Some typical examples where this arises include:
\begin{itemize}
\item \textbf{Image classification.} Given a large number of unlabeled images, we generally have very limited human resources with which to label them.
A more practical strategy is labeling a subset of the images and applying semi-supervised learning methods to classify the remaining unlabeled ones.
The challenge is to select which images to label in order to achieve the best classifier.
\item \textbf{Network management and information diffusion.} In a social network, advertisers with limited budgets need to select influential users in order to disseminate advertising and promotional information efficiently.
The challenge is how to identify and select users to improve and accelerate the spread of information.
\item \textbf{Smart buildings.} In a smart building we need to distribute sensors, e.g., surveillance cameras or environmental sensors, with a limited budget.
If each sensor can only detect anomalies in a small region, the challenge is how to assign their locations to maximize their coverage.
\end{itemize}
We can view these and similar problems as ``active manifold learning'' problems~\cite{zhu2003combining,xu2015active}, where the goal is to select representative landmarks for semi-supervised manifold learning.

In this paper, we propose a novel landmarking algorithm combining geometric landmarking methods with algebraic ones.
Specifically, we first give a bound on the learning error of semi-supervised manifold learning based on the manifold's alignment matrix.
Then, we show that many existing methods actually minimize the bound via different but one-sided strategies, which can be unified into a common algorithmic framework.
We propose a computationally-friendly surrogate for the error bound based on the Gershgorin circle theorem~\cite{gershgorin1931uber}, which is used as the objective function for active manifold learning.
We then propose a heuristic but effective landmark selection algorithm, which selects landmarks via deleting and updating the Gershgorin circles iteratively, where the indices of deleted circles corresponds to landmarks.
We analyze the complexity, scalability, and robustness of our algorithm and demonstrate its superiority to the existing state-of-art landmarking methods.

The contributions of our work are three-fold.
First, we explore properties of manifold learning and semi-supervised manifold learning and propose a unified framework of manifold landmarking for active manifold learning.
Second, in the proposed framework, we analyze existing manifold landmarking methods in depth and propose an active manifold learning algorithm that can be viewed as a generalization and unification of existing methods.
Third, we propose a Gershgorin circle-based landmarking algorithm with low computational complexity and high scalability that achieves encouraging results in both regression and classification tasks.

The remainder of this paper is organized as follows.
We first introduce related work and background on manifold learning, semi-supervised manifold learning, and active learning in Section~\ref{sec:relatedwork}.
Section~\ref{sec:land} provides an analysis of existing manifold landmarking methods and constructs our unified framework.
Section~\ref{sec:unified} contains a derivation of our proposed method for active manifold learning based on this unified framework and a comparison of this method to existing approaches.
Experiments and discussion are provided in Section~\ref{sec:exp}.
Finally, Section~\ref{sec:con} concludes the paper.
The Appendix contains additional technical details.

\section{Related Work and Background}\label{sec:relatedwork}
\subsection{Manifold learning}
Manifold models arise in a wide variety of signal processing and machine learning problems, and manifold learning serves as an important tool in applications such as computer vision and imaging~~\cite{DonohG_Image,WeinbS_Unsupervised,HintoDR_Modelling,bachmann2005exploiting}, array signal processing~\cite{schmidt1992multilinear,belloni2007doa,wu2010acoustic}, and graph-based signal analysis~\cite{shuman2013emerging,sandryhaila2013discrete,CostaH_Geodesic}, just to name a few.
Typical manifold learning methods include the locally linear embedding (LLE)~\cite{roweis2000nonlinear}, the local tangent space alignment (LTSA)~\cite{zhang2004principal}, the ISOMAP~\cite{tenenbaum2000global} method, the Laplacian Eigenmap (LE)~\cite{belkin2003laplacian,vepakomma2015fast} and the diffusion map~\cite{coifman2006diffusion,talmon2013diffusion}.
These methods can be unified into the framework of a common eigenvalue problem~\cite{yang2006semi,yang2010local,xu2015active}.
Specifically, suppose the manifold $\mathcal{X}$ is a low-dimensional surface embedded in a high-dimensional space, and its samples $\bm{X}=[\bm{x}_1,...,\bm{x}_N]\in\mathbb{R}^{D\times N}$ are the high-dimensional observations of the points on the manifold.
The manifold learning methods above find a low-dimensional representation of $\bm{X}$, denoted as $\bm{Y}=[\bm{y}_1,...,\bm{y}_N]\in\mathbb{R}^{d\times N}$ ($d\ll D$), by solving
\begin{eqnarray}\label{ML}
\begin{aligned}
\min_{\bm{Y}}&~\mbox{tr}(\bm{Y}\bm{\Phi}\bm{Y}^{T})\\
\text{s.t.}&~\bm{Y}\bm{Y}^{T}=\bm{I}_d.
\end{aligned}
\end{eqnarray}
Here $(\cdot)^T$ is the transpose of matrix, $\mbox{tr}(\cdot)$ calculates the trace of matrix, and $\bm{I}_d$ is a $d\times d$ identity matrix.
The matrix $\bm{\Phi}\in \mathbb{R}^{N\times N}$ is defined on a $K$-nearest neighbors ($K$-NN) graph derived from $\bm{X}$.
$\bm{\Phi}$ can be the Laplacian graph in LE, a variant of the Laplacian graph in diffusion maps, or the alignment matrix in ISOMAP, LLE, or LTSA.
In this paper, we call $\bm{\Phi}$ the alignment matrix.\footnote{Note that in this paper we focus on the manifold $\mathcal{X}$ embedded in $\mathbb{R}^{D}$. However, we note that ultimately our landmarking algorithm is applicable to any manifold with alignment matrix $\bm{\Phi}$.}
The derivations of $\bm{\Phi}$ for various manifold learning methods are given in Appendix A.

\subsection{Semi-supervised manifold learning}
Let $\bm{X} = [\bm{X}_{\mathcal{L}}, \bm{X}_{\bar{\mathcal{L}}} ]$, where
$\mathcal{L}$ is the index set of the labeled samples with cardinality $|\mathcal{L}|=L$ and $\bar{\mathcal{L}}=\{1,...,N\}\setminus \mathcal{L}$ is the index set of the unlabeled samples.
Given $\bm{X}_{\mathcal{L}}$ and labels $\bm{Z}_{\mathcal{L}}$, the goal of semi-supervised manifold learning is to determine the labels $\bm{Z}_{\bar{\mathcal{L}}}$ of the unlabeled samples $\bm{X}_{\bar{\mathcal{L}}}$~\cite{belkin2006manifold,yang2006semi,zhang2008spectral,BelkiN_Semi}.
To achieve this aim, the Least Squares (LS) method in~\cite{yang2006semi} learns the mapping from the ambient space $\mathcal{X}$ to the label space $\mathcal{Z}$ directly by solving
\begin{eqnarray}\label{LS}
\begin{aligned}
\min_{\bm{Z}_{\bar{\mathcal{L}}}}~\mbox{tr}\left(
[\bm{Z}_{\mathcal{L}},\bm{Z}_{\bar{\mathcal{L}}}]
\begin{bmatrix}
\bm{\Phi}_{\mathcal{LL}} &\bm{\Phi}_{\mathcal{L}\bar{\mathcal{L}}}\\
\bm{\Phi}_{\bar{\mathcal{L}}\mathcal{L}} &\bm{\Phi}_{\bar{\mathcal{L}}\bar{\mathcal{L}}}
\end{bmatrix}
\begin{bmatrix}
\bm{Z}_{\mathcal{L}}^T\\
\bm{Z}_{\bar{\mathcal{L}}}^T
\end{bmatrix}
\right)+\underbrace{\gamma\|\bm{Z}_{\bar{\mathcal{L}}}\|_F^{2}}_{\small\text{optional}},
\end{aligned}
\end{eqnarray}
where the first term of (\ref{LS}) enforces a manifold structure on $\bm{Z}$ (estimated from $\bm{X}$) and the second term of (\ref{LS}) is an optional regularizer on the Frobenius norm of $\bm{Z}_{\bar{\mathcal{L}}}$.

The spectral method ({Spec}) in~\cite{zhang2008spectral} assumes that both the data manifold $\mathcal{X}$ and the label manifold $\mathcal{Z}$ are different images of the same latent space: $\mathcal{X} = h(\mathcal{Y})$ and $\mathcal{Z} = g(\mathcal{Y})$, and the mapping $g:~\mathcal{Y}\rightarrow \mathcal{Z}$ is an affine transformation.
The $\bm{Y}\subset \mathcal{Y}$ is learned by traditional manifold learning algorithm with a label-based regularizer:
\begin{eqnarray}\label{Spec}
\begin{aligned}
\min_{\bm{Y}}&~\mbox{tr}(\bm{Y\Phi Y}^{T})+\gamma\mbox{tr}(\bm{Y}_{\mathcal{L}}\bm{G}\bm{Y}_{\mathcal{L}}^{T})\\
\text{s.t.}&~\bm{YY}^{T}=\bm{I}_d,
\end{aligned}
\end{eqnarray}
where $\bm{G}$ is the orthogonal projection whose null space is spanned by $[\bm{1}, \bm{Z}_{\mathcal{L}}^{T}]$.
After obtaining $\bm{Y}$, we then learn an affine transformation between $\bm{Y}_{\mathcal{L}}$ and $\bm{Z}_{\mathcal{L}}$.

It should be noted that besides learning labels of samples, in more general cases the latent variables learned by semi-supervised manifold learning can be used as inputs/features to traditional learning algorithms~\cite{belkin2006manifold,zheng2011graph,xu2015dictionary}.

\begin{algorithm}[!t]
  \caption{Active Manifold Learning}
  \label{alg2}
  \begin{algorithmic}[1]
    \Require $\bm{X}=[\bm{x}_i]\in\mathbb{R}^{D\times N}$, the number of landmarks $L$.
    \Ensure Labels $\bm{Z}=[\bm{Z}_{\mathcal{L}}, \bm{Z}_{\bar{\mathcal{L}}}]\in\mathbb{R}^{d\times N}$.
    \State Generate $\bm{\Phi}$ via any manifold learning algorithm.
    \State Apply a landmark selection algorithm to choose $\mathcal{L}$.
    \State Label $\{\bm{x}_i\}_{i\in\mathcal{L}}$ with $\bm{Z}_{\mathcal{L}}$.
    \State \emph{For learning labels:} \\
    \quad\quad Apply SSML to $\{\bm{X},\bm{Z}_{\mathcal{L}}\}$.
    \State \emph{For manifold-regularized tasks:} \\
    \quad\quad Learn features $\bm{Y}$ and train model with $\{\bm{Y}_{\mathcal{L}}, \bm{Z}_{\mathcal{L}}\}$.
    \State Return estimated labels $\bm{Z}_{\bar{\mathcal{L}}}$ for $\{\bm{x}\}_{i\in \bar{\mathcal{L}}}$.
  \end{algorithmic}
\end{algorithm}

\subsection{Active learning}
Active learning~\cite{cohn1996active,settles2010active,liang2015gpm} has been used to select representative samples~\cite{schohn2000less,paisley2010active,avrachenkov2013choice,tsiligkaridis2015decentralized} and improve learning results in many applications ranging from computer vision~\cite{vijayanarasimhan2014large}, natural language processing~\cite{thompson1999active}, speech recognition~\cite{riccardi2005active}, data mining~\cite{zhao2013active}, and geoscience~\cite{chen2006improved,chi2013selection}, and many more.

From the viewpoint of active learning~\cite{zhu2003combining,chi2014active}, the challenge of active manifold learning is to select which samples on the manifold to label in order to minimize the learning error on the remaining samples.
Following~\cite{yang2006semi,xu2015active}, our approach to active manifold learning is to combine a landmark selection algorithm with semi-supervised manifold learning or manifold regularization. This is summarized in Algorithm~\ref{alg2}.

This problem is very close to manifold landmarking, where the aim is to select representative samples on a manifold.
Many methods have been proposed with different motivations, and accordingly, they apply different selection criteria.
The Nystr{\"o}m method~\cite{drineas2005nystrom,drineas2006fast} and its variants~\cite{zhang2008improved,kumar2012sampling} aim to achieve a good low-rank approximation of the kernel matrix of samples by selecting its columns and rows.
The method of optimal experimental design~\cite{pukelsheim1993optimal} and the volume sampling method~\cite{derezinski2017unbiased} are methods that aim to improve performance for (generalized) linear regression.
The method in~\cite{de2004sparse,silva2006selecting} aims to maximize the minimum geodesic distance between landmarks.
The work in~\cite{wachinger2015diverse} achieves a scalable landmarking method based on determinantal point processes.

The concerns that arise in the semi-supervised setting are not necessarily limited to simply preserving the geometric structure -- nevertheless, we will see that manifold landmarking methods can play an important role in the active manifold learning problem.
Below we will describe two broad categories of landmarking methods which we denote {\em algebraic} and {\em geometric} methods, and then show that they can be unified into a common framework for active manifold learning.
Compared with our previous work~\cite{xu2015active} and other existing  methods, the proposed landmarking method in this paper considers both the algebraic aim and the geometric aim of manifold landmarking and applies a new objective function, which achieves at least comparable learning results despite having much lower computational complexity.

\section{A Unified Framework for Manifold Landmarking}\label{sec:land}
\subsection{Algebraic methods}\label{ssec1}
In~\cite{xu2015active} we recently proposed a manifold landmarking method (called \textbf{MinCond}) based on an algebraic analysis of semi-supervised manifold learning that aims to approximately minimize the condition number of the alignment matrix. This method is motivated by the approaches to semi-supervised manifold learning in (\ref{LS}) and (\ref{Spec}). In particular, if we denote the objective function in (\ref{LS}) as $f(\bm{Z})$ and set the gradient of $f(\bm{Z})$ with respect to the labels $\bm{Z}_{\bar{\mathcal{L}}}$ to be zero, i.e., $\frac{\partial f(\bm{Z})}{\partial \bm{Z}_{\bar{\mathcal{L}}}}=\bm{0}$, we can obtain a closed-form solution for $\bm{Z}_{\bar{\mathcal{L}}}$ by solving the following linear system of equations:
\begin{eqnarray}\label{LinearSys}
\begin{aligned}
(\bm{\Phi}_{\bar{\mathcal{L}}\bar{\mathcal{L}}} +\gamma\bm{I}_{N-L})\bm{Z}_{\bar{\mathcal{L}}}^{T}
=\bm{\Phi}_{\bar{\mathcal{L}}\mathcal{L}}\bm{Z}_{\mathcal{L}}^{T}.
\end{aligned}
\end{eqnarray}
One can also consider the Lagrangian function of (\ref{Spec}), i.e.,
\begin{equation} \label{eq:lagrangian}
\mbox{tr}(\bm{Y\Phi Y}^T)+\gamma \mbox{tr}(\bm{Y}_{\mathcal{L}}\bm{G}\bm{Y}_{\mathcal{L}}^T)+\mbox{tr}(\bm{B}(\bm{YY}^T-\bm{I}_d)),
\end{equation}
where $\bm{B}\in\mathbb{R}^{d\times d}$ contains the Lagrange multipliers. Note that by symmetry we can decompose $\bm{B}$ as $\sum_{i=1}^{d}\bm{\beta}_i\bm{\beta}_i^{T}$, where $\bm{\beta}_i\in\mathbb{R}^d$ for $i=1,\ldots,d$, and rewrite~\eqref{eq:lagrangian} as
\[
\mbox{tr}(\bm{Y\Phi Y}^T)+\gamma \mbox{tr}(\bm{Y}_{\mathcal{L}}\bm{G}\bm{Y}_{\mathcal{L}}^T)+\sum_{i=1}^{d}\bm{\beta}_i^T\bm{YY}^T\bm{\beta}_i-\bm{\beta}_i^T\bm{\beta}_i.
\]
Setting the gradient with respect to the unknown latent variables $\bm{Y}_{\bar{\mathcal{L}}}$ to be zero yields
\begin{eqnarray}\label{LinearSys2}
\begin{aligned}
(\bm{\Phi}_{\bar{\mathcal{L}}\bar{\mathcal{L}}})\bm{Y}_{\bar{\mathcal{L}}}^{T}+2\bm{Y}_{\bar{\mathcal{L}}}^{T}\sum_{i=1}^{d}(\bm{\beta}_i\bm{\beta}_i^T)=\bm{\Phi}_{\bar{\mathcal{L}}\mathcal{L}}\bm{Y}_{\mathcal{L}}^{T}.
\end{aligned}
\end{eqnarray}

Note that if we ignore the effect of the optional regularizer and the Lagrange multipliers (i.e., set $\gamma$ and the $\bm{\beta}_i$, which can typically be set quite small, to zero), then (\ref{LinearSys2}) is equivalent to (\ref{LinearSys}).
Thus, although the analysis below is for (\ref{LS}), it can also apply to (\ref{Spec}) by simply replacing $\bm{Z}_{\bar{\mathcal{L}}}$ with $\bm{Y}_{\bar{\mathcal{L}}}$ throughout.

In practice, we typically expect the observations $\bm{X}$ to be somewhat noisy.  In this case, we can treat the corresponding alignment matrix $\bm{\Phi}$ as being also contaminated with noise, in which case (\ref{LinearSys}) becomes
\begin{eqnarray}\label{error}
\begin{aligned}
(\bm{\Phi}_{\bar{\mathcal{L}}\bar{\mathcal{L}}}+\bm{E}_2)\widehat{\bm{Z}}_{\bar{\mathcal{L}}}^{T}=
(\bm{\Phi}_{\bar{\mathcal{L}}\mathcal{L}}+\bm{E}_1)\bm{Z}_{\mathcal{L}}^{T},
\end{aligned}
\end{eqnarray}
where $\bm{E}_1$, $\bm{E}_2$ are noise matrices and
$\widehat{\bm{Z}}_{\bar{\mathcal{L}}}$ is our estimate of the $\bm{Z}_{\bar{\mathcal{L}}}$ that one would obtain using the ``noise-free'' $\bm{\Phi}$.

As shown in~\cite{golub2012matrix,xu2015active}, the relative error between our estimate $\widehat{\bm{Z}}_{\bar{\mathcal{L}}}$ and the ``noise-free'' estimate $\bm{Z}_{\bar{\mathcal{L}}}$ is bounded by
\begin{eqnarray}\label{bound}
\begin{aligned}
\frac{\|\bm{Z}_{\bar{\mathcal{L}}}-\widehat{\bm{Z}}_{\bar{\mathcal{L}}}\|_2}{\|\bm{Z}_{\bar{\mathcal{L}}}\|_2} &\leq
\kappa(\bm{\Phi}_{\bar{\mathcal{L}}\bar{\mathcal{L}}})\left( \frac{\|\bm{E}_{1}\|_2}{\|\bm{\Phi}_{\bar{\mathcal{L}}\mathcal{L}}\|_2}+\frac{\|\bm{E}_{2}\|_2}{\|\bm{\Phi}_{\bar{\mathcal{L}}\bar{\mathcal{L}}}\|_2}\right)\\
&\leq
\epsilon \, \kappa(\bm{\Phi}_{\bar{\mathcal{L}}\bar{\mathcal{L}}}) \left( \frac{1}{\|\bm{\Phi}_{\bar{\mathcal{L}}\mathcal{L}}\|_2}+\frac{1}{\|\bm{\Phi}_{\bar{\mathcal{L}}\bar{\mathcal{L}}}\|_2}\right),
\end{aligned}
\end{eqnarray}
where $\|\cdot\|_2$ is the induced $\ell_2$ matrix norm, $\kappa( \cdot )$ computes the condition number, and $\epsilon = \max(\|\bm{E}_{1}\|_2,\|\bm{E}_{2}\|_2)$.
Because the relative learning error is directly related to $\kappa(\bm{\Phi}_{\bar{\mathcal{L}}\bar{\mathcal{L}}})$, {MinCond} aims to select landmarks by deleting $L$ rows/columns of $\bm{\Phi}$ so that the remaining principal submatrix $\bm{\Phi}_{\bar{\mathcal{L}}\bar{\mathcal{L}}}$ has the smallest possible condition number:
\begin{eqnarray}\label{opt1}
\begin{aligned}
\min_{\mathcal{L}}&~\kappa(\bm{\Phi}_{\bar{\mathcal{L}}\bar{\mathcal{L}}})\\
\text{s.t.}&~|\mathcal{L}|=L.
\end{aligned}
\end{eqnarray}

Traditional condition number minimization algorithms such as ~\cite{braatz1994minimizing,greif2006minimizing,lu2011minimizing,chen2011minimizing} require the feasible domain to be a compact convex set of the cone of positive semidefinite matrices, which is not available for (\ref{opt1}).
Generally, (\ref{opt1}) can be solved approximately by the Rank-revealing QR-factorization (RRQR) in~\cite{hong1992rank}.
In Appendix B, we demonstrate that RRQR can give an upper bound on the solution of (\ref{opt1}).
However, the bound is too loose for practical application.
In practice, MinCond reformulates the problem from minimizing the condition number of the alignment matrix to minimizing the dynamic range of the eigenvalues of the logarithmic alignment matrix and deletes the rows/columns of the alignment matrix iteratively.
Specifically, the logarithmic version of the objective function in (\ref{opt1}) is
\begin{eqnarray}\label{logactive}
\begin{aligned}
|\ln{\lambda_{\max}(\bm{\Phi}_{\bar{\mathcal{L}}\bar{\mathcal{L}}})}-\ln{\lambda_{\min}(\bm{\Phi}_{\bar{\mathcal{L}}\bar{\mathcal{L}}})}|.
\end{aligned}
\end{eqnarray}
Here, $\ln{\lambda_{\max}}$ and $\ln{\lambda_{\min}}$ are the largest and the smallest eigenvalues of $\ln{(\bm{\Phi}_{\bar{\mathcal{L}}\bar{\mathcal{L}}})}$.
Instead of minimizing the dynamic range of the eigenvalues directly, we minimize an upper bound.
Specifically, the objective function becomes
\begin{eqnarray}\label{final}
\begin{aligned}
|\Lambda_{u}(\bm{\Phi}_{\bar{\mathcal{L}}\bar{\mathcal{L}}})
-\Lambda_{l}(\bm{\Phi}_{\bar{\mathcal{L}}\bar{\mathcal{L}}})|,
\end{aligned}
\end{eqnarray}
where $\Lambda_u(\bm{\Phi}_{\bar{\mathcal{L}}\bar{\mathcal{L}}})$ and $\Lambda_l(\bm{\Phi}_{\bar{\mathcal{L}}\bar{\mathcal{L}}})$ are upper and lower bounds on the eigenvalues of $\ln{(\bm{\Phi}_{\bar{\mathcal{L}}\bar{\mathcal{L}}})}$, which are computed according to the Gershgorin circles of $\ln{(\bm{\Phi}_{\bar{\mathcal{L}}\bar{\mathcal{L}}})}$.
By deleting the Gershgorin circles of $\ln{(\bm{\Phi}_{\bar{\mathcal{L}}\bar{\mathcal{L}}})}$ iteratively, we can shrink the interval $[\Lambda_u, \Lambda_l]$, which bounds the condition number accordingly.
The samples corresponding to deleted circles are the selected landmarks.  Further details of MinCond can be found in~\cite{xu2015active}.

\subsection{Geometric methods}\label{ssec2}
Many other manifold landmarking methods have a more geometric flavor.
Representative methods include a geodesic distance-based algorithm ({\textbf{MaxMinGeo}})~\cite{de2004sparse,silva2006selecting} and an approximate determinantal point process-based algorithm ({\textbf{ApproxDPP}})~\cite{wachinger2015diverse}.
{MaxMinGeo} and {ApproxDPP} aim to distribute the landmarks to maximize the coverage of the landmarks on the target manifold by ensuring that the landmarks are not too close in the ambient space.
Denote the distance (under a given metric) between samples $\bm{x}_i$ and $\bm{x}_j$ as $d_{ij}$.
Mathematically, the problem of maximizing the minimum distance between landmarks can be written as
\begin{eqnarray}\label{mmgeo}
\begin{aligned}
\max_{\mathcal{L}}&~\min_{i,j\in \mathcal{L}}~d_{ij}\\
\text{s.t.}&~|\mathcal{L}|=L,
\end{aligned}
\end{eqnarray}
where $\mathcal{L}$ is the set of landmarks.
MaxMinGeo tries to solve this problem approximately in a heuristic way.
In particular, MaxMinGeo first initializes several landmarks (or one landmark) randomly, and then it adds new landmarks iteratively and ensures that the minimum geodesic distance between the new landmark and existing ones is maximized.

To further accelerate manifold landmarking, an approximate but scalable DPP method (ApproxDPP) is proposed in~\cite{wachinger2015diverse}, which makes an additional concession. Instead of maximizing the minimum distance between two arbitrary landmarks, this DPP-based method simply ensures that
the selection of the new landmark is performed using a probabilistic distribution that suppresses the probability of selecting existing landmarks' neighbors.
As a result, the landmarks will tend to not be neighbors of each other.
In Section~\ref{sec:exp}, we will show that the ApproxDPP method achieves comparable performance to MaxMinGeo using much less runtime, which can be viewed as an approximate but fast implementation of {MaxMinGeo}.

Both of these methods can be viewed as heuristics for attempting to approximately maximize the coverage of the landmarks on the target manifold by ensuring that once a landmark is selected, we use the opportunity to select additional landmarks to gain a higher degree of coverage of the manifold by avoiding the immediate neighbors of the landmarks selected up to that point.
When the distance between samples is defined on a $K$-NN graph of samples, the strategy of these methods can be re-interpreted based on the alignment matrix.
Specifically, the alignment matrices in LE, LLE, and LTSA have the following Property, which we prove in Appendix C:
\begin{quote}
\textbf{Property 1.} For the $\bm{\Phi}=[\phi_{ij}]$ in LE, LLE, and LTSA, $\phi_{ij}\neq 0$ if and only if samples $\bm{x}_i$ and $\bm{x}_j$ are neighbors in the $K$-NN graph.
\end{quote}
For example, LE uses a Laplacian graph matrix as the alignment matrix.
The entry $\phi_{ij}=-d_{ij}$ if samples $\bm{x}_i$ and $\bm{x}_j$ are neighbors, otherwise $\phi_{ij}=0$.  $\phi_{ii}=\sum_{j}d_{ij}$.

In this setting, the implicit goal of geometric methods that the landmarks should maximize the coverage of the landmarks on the target manifold can be achieved when the submatrix $\bm{\Phi}_{\bar{\mathcal{L}}\mathcal{L}}$ has as many nonzeros as possible.
In other words, we would like to ensure that landmarks in $\mathcal{L}$ have many direct connections in the $K$-NN graph to elements in $\bar{\mathcal{L}}$.  One way to promote this objective is to solve the optimization problem
\begin{eqnarray}\label{impg}
\begin{aligned}
\max_{\mathcal{L}}&~\max_{j \in \mathcal{L}} \|\bm{\phi}_j\|_0\\
\text{s.t.}&~|\mathcal{L}|=L,
\end{aligned}
\end{eqnarray}
where $\bm{\phi}_j$ is the column of $\bm{\Phi}_{\bar{\mathcal{L}}\mathcal{L}}$ corresponding to the index $j$ and $\|\cdot\|_0$ is the so-called ``$\ell_0$ norm'', which counts the number of nonzero elements in a vector.
This objective function encourages the selection of landmarks which are densely connected to elements in $\bar{\mathcal{L}}$.

Unfortunately, solving~\eqref{impg} directly is intractable.
Additionally, in some cases there can be many solutions to~\eqref{impg} with widely varying levels of coverage.
For example, instead of defining a $K$-NN graph, suppose we define the neighbors of a point as those with distances below a certain threshold.
Suppose that the most connected point has $M$ neighbors.
If $N-M>L-1$, then any set of landmarks containing that point and $L-1$ arbitrary points disconnected with it is a solution to~\eqref{impg}.
To address these limitations, in this work we further relax the objective function from the $\ell_0$-norm to the $\ell_1$ matrix norm, i.e., $\|\bm{\Phi}_{\bar{\mathcal{L}}\mathcal{L}}\|_1=\max_{j\in\mathcal{L}}\sum_{i\in\bar{\mathcal{L}}}|\phi_{ij}|$.
Finally, note that
\begin{eqnarray}\label{objG}
\begin{aligned}
\max_{\mathcal{L}}~\|\bm{\Phi}_{\bar{\mathcal{L}}\mathcal{L}}\|_1
\Leftrightarrow
\min_{\mathcal{L}}~\frac{1}{\|\bm{\Phi}_{\bar{\mathcal{L}}\mathcal{L}}\|_1}.
\end{aligned}
\end{eqnarray}

\subsection{A unified algorithmic framework}
As seen above, different viewpoints on landmarking methods actually lead to very different criteria for landmark selection.
On the one hand, the algebraic method focuses on minimizing learning error for the remaining samples by minimizing the condition number of the remaining principal submatrix (i.e., MinCond).
On the other hand, the geometric methods focus on maximizing the diversity of the landmarks.
Under certain metrics, the minimum distance between landmarks is maximized deterministically (i.e., MaxMinGeo) or probabilistically (i.e., DPP) to ensure that the landmarks have a good coverage on the target manifold.
Based on the analysis in Section~\ref{ssec1} and~\ref{ssec2}, we can unify these two kinds of methods into a single algorithmic framework.

Towards this end, we first note that we can further bound the right-hand side of (\ref{bound}) via the standard norm inequality for an $M\times N$ matrix $\bm{A}$ of $\|\bm{A}\|_1\le\sqrt{M}\|\bm{A}\|_2$, yielding
\begin{equation} \label{eq:l1bound}
\tfrac{\|\bm{Z}_{\bar{\mathcal{L}}}-\widehat{\bm{Z}}_{\bar{\mathcal{L}}}\|_2}{\|\bm{Z}_{\bar{\mathcal{L}}}\|_2} \leq \epsilon \, \kappa(\bm{\Phi}_{\bar{\mathcal{L}}\bar{\mathcal{L}}}) \sqrt{N-L}  \left( \tfrac{1}{\|\bm{\Phi}_{\bar{\mathcal{L}}\mathcal{L}}\|_1}+\tfrac{1}{\|\bm{\Phi}_{\bar{\mathcal{L}}\bar{\mathcal{L}}}\|_1}\right).
\end{equation}
When we consider~\eqref{eq:l1bound} in place of \eqref{bound}, we observe that both the algebraic and geometric methods can be viewed as attempting to minimize different parts of the same bound. MinCond aims to minimize the $\kappa(\bm{\Phi}_{\bar{\mathcal{L}}\bar{\mathcal{L}}})$ term in~\eqref{eq:l1bound}, while {MaxMinGeo} and {ApproxDPP} can be viewed as implicitly minimizing the $\frac{1}{\|\bm{\Phi}_{\bar{\mathcal{L}}\mathcal{L}}\|_1}$ term in~\eqref{eq:l1bound}. However, according to the bound in~\eqref{eq:l1bound}, the learning error is determined not only by these terms alone, but by their combination, and further also by $\frac{1}{\|\bm{\Phi}_{\bar{\mathcal{L}}\mathcal{L}}\|_1}+\frac{1}{\|\bm{\Phi}_{\bar{\mathcal{L}}\bar{\mathcal{L}}}\|_1}$.
In our view, the entire right side of~\eqref{eq:l1bound}, which considers all of these criteria simultaneously, provides us with a more natural and reasonable criterion for manifold landmarking in the context of active learning.
In particular, we can achieve manifold landmarking by solving
\begin{eqnarray}\label{Propose1}
\begin{aligned}
\min_{\mathcal{L}}&~\kappa(\bm{\Phi}_{\bar{\mathcal{L}}\bar{\mathcal{L}}})\left(\frac{1}{\|\bm{\Phi}_{\bar{\mathcal{L}}\mathcal{L}}\|_1}+\frac{1}{\|\bm{\Phi}_{\bar{\mathcal{L}}\bar{\mathcal{L}}}\|_1}\right)\\
\text{s.t.}&~|\mathcal{L}|=L.
\end{aligned}
\end{eqnarray}
Here, the objective function in (\ref{Propose1}) is the right side of inequality (\ref{bound}).
The first term $\kappa(\bm{\Phi}_{\bar{\mathcal{L}}\bar{\mathcal{L}}})\frac{1}{\|\bm{\Phi}_{\bar{\mathcal{L}}\mathcal{L}}\|_1}$ corresponds to a combination of algebraic and geometric manifold lankmarking methods.
The second term $\kappa(\bm{\Phi}_{\bar{\mathcal{L}}\bar{\mathcal{L}}})\frac{1}{\|\bm{\Phi}_{\bar{\mathcal{L}}\bar{\mathcal{L}}}\|_1}$ can be viewed as a regularizer.
Specifically, $\|\bm{\Phi}_{\bar{\mathcal{L}}\bar{\mathcal{L}}}\|_1=\max_{j\in\bar{\mathcal{L}}}\sum_{i\in\bar{\mathcal{L}}}|\phi_{ij}|=\max_{j\in\bar{\mathcal{L}}}\phi_{jj}+\sum_{i\in\bar{\mathcal{L}}\setminus j}d_{ij}$, which involves the sum of distances and that of connections between each unlabeled sample to the remaining unlabeled samples.
By trying to minimize $\frac{1}{\|\bm{\Phi}_{\bar{\mathcal{L}}\bar{\mathcal{L}}}\|_1}$, we ensure that there is at least one unlabeled sample densely connecting with the remaining unlabeled samples.
From the viewpoint of graph-based label propagation~\cite{zhu2009introduction}, those densely-connected unlabeled samples should aid in subsequent learning tasks --- the labels can spread from those labeled samples quickly as long as those densely-connected unlabeled samples are assigned labels. 
The importance of the regularizer is controlled by the condition number $\kappa(\bm{\Phi}_{\bar{\mathcal{L}}\bar{\mathcal{L}}})$.

\section{Gershgorin Circle-Based Landmark Selection}\label{sec:unified}

\subsection{Proposed algorithm}\label{ssec:alg}
Optimizing (\ref{Propose1}) directly is intractable.
Instead, we propose a heuristic but very effective algorithm for approximately solving the problem.
In particular, our algorithm involves two key steps.

\textbf{1) Constructing a surrogate objective function.} Given an alignment matrix $\bm{\Phi}$, we construct a symmetric matrix $\bm{\Psi}=\bm{\Phi}+\alpha\bm{I}_{N}$ as a regularized alignment matrix, where $\alpha > 0$, and replace the objective function in (\ref{Propose1}) with
\begin{eqnarray}\label{Propose2}
\begin{aligned}
\kappa(\bm{\Psi}_{\bar{\mathcal{L}}\bar{\mathcal{L}}})\left(\frac{1}{\|\bm{\Psi}_{\bar{\mathcal{L}}\mathcal{L}}\|_1}+\frac{1}{\|\bm{\Psi}_{\bar{\mathcal{L}}\bar{\mathcal{L}}}\|_1}\right).
\end{aligned}
\end{eqnarray}
By replacing $\bm{\Phi}$ with $\bm{\Psi}$ we can ensure that the smallest eigenvalue of $\bm{\Psi}$ is well-separated from 0, which as we will see below is key to our approach.
We can make this substitution because of the following property.
\begin{quote}
\textbf{Property 2.} The results of the manifold learning approach in (\ref{ML}) or the semi-supervised manifold learning approaches in (\ref{LS}) and (\ref{Spec}) are not changed by replacing $\bm{\Phi}$ with $\bm{\Psi}$.
\end{quote}
\begin{proof}
For manifold learning, after replacing $\bm{\Phi}$ with $\bm{\Psi}$, we can rewrite the objective function in~(\ref{ML}) as
\begin{eqnarray*}
\begin{aligned}
\mbox{tr}(\bm{Y\Phi Y}^{T}+\alpha\bm{YY}^{T}).
\end{aligned}
\end{eqnarray*}
Because of the normalization constraint $\bm{YY}^{T}=\bm{I_d}$, we have $\mbox{tr}(\bm{YY}^{T})=\mbox{tr}(\bm{I}_d)=d$.
Therefore, replacing $\bm{\Phi}$ with $\bm{\Psi}$ only introduces a constant into the original objective function, which does not change the optimal point.

For the LS approach to semi-supervised manifold learning, upon replacing $\bm{\Phi}$ with $\bm{\Psi}$ the objective function in~(\ref{LS}) can be rewritten as:
\begin{eqnarray*}
\begin{aligned}
\mbox{tr}(\bm{Z\Phi Z}^{T})+\alpha\mbox{tr}(\bm{Z}_{\mathcal{L}}\bm{Z}_{\mathcal{L}}^{T})
+\alpha\mbox{tr}(\bm{Z}_{\bar{\mathcal{L}}}\bm{Z}_{\bar{\mathcal{L}}}^{T}).
\end{aligned}
\end{eqnarray*}
Note that $\mbox{tr}(\bm{Z}_{\bar{\mathcal{L}}}\bm{Z}_{\bar{\mathcal{L}}}^{T})=\|\bm{Z}_{\bar{\mathcal{L}}}\|_F^2$.
Compared with the original objective function in~(\ref{LS}), replacing $\bm{\Phi}$ with $\bm{\Psi}$ is equivalent to adding the optional regularizer of $\bm{Z}_{\bar{\mathcal{L}}}$ with $\gamma=\alpha$.

For the Spec method, the objective function in~(\ref{Spec}) can be rewritten as:
\begin{eqnarray*}
\begin{aligned}
\mbox{tr}(\bm{Y\Phi Y}^{T})+\alpha d+\gamma\mbox{tr}(\bm{Y}_{\mathcal{L}}\bm{G}\bm{Y}_{\mathcal{L}}^{T}).
\end{aligned}
\end{eqnarray*}
Again, replacing $\bm{\Phi}$ with $\bm{\Psi}$ does not change the optimal point.
\end{proof}

Following the work in~\cite{xu2015active}, we can further relax (\ref{Propose2}) with the help of the Gershgorin circle theorem~\cite{gershgorin1931uber}.
\begin{theorem}[Gershgorin circle theorem] \label{thm:Gersh}
Let $\bm{\Psi}$ be a $N\times N$ matrix with entries $\psi_{ij}$. For each eigenvalue $\lambda$ of $\bm{\Psi}$, there exists an $i\in\{1,...,N\}$ such that:
\[
|\lambda-\psi_{ii}|\leq r_i := \sum_{j\neq i}|\psi_{ij}|.
\]
Defining $c_i := \psi_{ii}$, the set $\mathcal{C}_i=\{ x~:~|x-c_i|\leq r_i \}$ is called the $i^{\text{th}}$ Gershgorin circle of $\bm{\Psi}$, where $C_i$ is the center and $r_i$ is the radius of the circle.
\end{theorem}

Denote the center and the radius of $\bm{\Psi}$'s $i^{\text{th}}$ Gershgorin circle as $c_i = \psi_{ii}$ and $r_{i}=\sum_{j\neq i}|\psi_{ij}|$, respectively, and denote the radius of $\bm{\Psi}_{\bar{\mathcal{L}}\bar{\mathcal{L}}}$'s $i^{\text{th}}$ Gershgorin circle as $s_{i}=\sum_{j\in \bar{\mathcal{L}}\setminus i}|\psi_{ij}|$.
From Theorem~\ref{thm:Gersh} and the fact that $\bm{\Psi}$ is symmetric, the eigenvalues of $\bm{\Psi}$ (and its arbitrary principal submatrix) are real and must fall within $\cup_{i}[c_i-r_i,c_i+r_i]$.
Moreover, we can always choose an $\alpha>0$ such that the lower bound $\min(c_i-r_i)$ is positive.
In such a situation, we obtain an upper bound for the condition number of $\bm{\Psi}$:
\begin{eqnarray}
\begin{aligned}
\kappa(\bm{\Psi})=\Bigl|\frac{\lambda_{\max}(\bm{\Psi})}{\lambda_{\min}(\bm{\Psi})}\Bigr|\leq\frac{\max(c_i+r_i)}{\min(c_i-r_i)}.
\end{aligned}
\end{eqnarray}
Moreover, note that $\|\bm{\Psi}\|_1=\max_{1\leq j\leq N}\sum_{i=1}^N |\psi_{ij}|=\max_{1\leq j\leq N}|c_j+r_j|$.
Similarly, we have $\|\bm{\Psi}_{\bar{\mathcal{L}}\bar{\mathcal{L}}}\|_1=\max_{i\in\bar{\mathcal{L}}}(c_i+s_i)$ and $\|\bm{\Psi}_{\bar{\mathcal{L}}\mathcal{L}}\|_1=\max_{i\in\bar{\mathcal{L}}}(r_i-s_i)$.
Therefore, we have
\begin{eqnarray*}\label{Upper}
\begin{aligned}
&\kappa(\bm{\Psi}_{\bar{\mathcal{L}}\bar{\mathcal{L}}})\left(\frac{1}{\|\bm{\Psi}_{\bar{\mathcal{L}}\mathcal{L}}\|_1}+\frac{1}{\|\bm{\Psi}_{\bar{\mathcal{L}}\bar{\mathcal{L}}}\|_1}\right)\\
&\leq \frac{\max_{i\in\bar{\mathcal{L}}}(c_i+s_i)}{\min_{i\in\bar{\mathcal{L}}}(c_i-s_i)}
\left(
\frac{\|\bm{\Psi}_{\bar{\mathcal{L}}\mathcal{L}}\|_1+\|\bm{\Psi}_{\bar{\mathcal{L}}\bar{\mathcal{L}}}\|_1}{\|\bm{\Psi}_{\bar{\mathcal{L}}\mathcal{L}}\|_1 \|\bm{\Psi}_{\bar{\mathcal{L}}\bar{\mathcal{L}}}\|_1}
\right)\\
&=\frac{\max_{i\in\bar{\mathcal{L}}}(c_i+s_i)}{\min_{i\in\bar{\mathcal{L}}}(c_i-s_i)}
\frac{\max_{i\in\bar{\mathcal{L}}}(r_i-s_i)+\max_{i\in\bar{\mathcal{L}}}(c_i+s_i)}{\max_{i\in\bar{\mathcal{L}}}(r_i-s_i)\max_{i\in\bar{\mathcal{L}}}(c_i+s_i)}\\
&=\frac{\max_{i\in\bar{\mathcal{L}}}(r_i-s_i)+\max_{i\in\bar{\mathcal{L}}}(c_i+s_i)}{\min_{i\in\bar{\mathcal{L}}}(c_i-s_i)\max_{i\in\bar{\mathcal{L}}}(r_i-s_i)}\\
&=Q(\bar{\mathcal{L}}).
\end{aligned}
\end{eqnarray*}

$Q(\bar{\mathcal{L}})$ gives an upper bound for our error that can be used as a surrogate for our objective function in (\ref{Propose1}).\footnote{We note that there is an alternative derivation which results in the exact same function $Q(\bar{\mathcal{L}})$ as an upper bound on the learning error. In particular, if one begins in~\eqref{bound} by establishing a bound on the learning error in the $\ell_1$ norm, then one can obtain a similar bound by bounding $\kappa_1$ using the fact that for any $N \times N$ matrix $\bm{A}$, $\|\bm{A}^{-1}\|_1 \le N/ \min_{i} ||a_{ii}| - r_i(\bm{A})|$, where $r_i$ is defined as above. This is a consequence of standard norm inequalities combined with the Ahlberg–Nilson–Varah~\cite{Ahlberg1963convergence,Varah1975lower} bound on $\|\bm{A}^{-1}\|_{\infty}$.}

\textbf{2) Deleting and updating Gershgorin circles.}
The second key step of our algorithm is deleting and updating Gershgorin circles iteratively, such that in each step, the surrogate function is reduced in a greedy fashion.
Specifically, we propose our Gershgorin circle-based landmark selection ({\bf GCLS}) algorithm in Algorithm~\ref{alg1}.
The configuration of $\alpha$ ensures that $b_{\min}$ is nonnegative.
After initialization, we select landmarks by deleting Gershgorin cricles (i.e., the rows/columns of $\bm{\Psi}$) iteratively (Line 7-13).
At each iteration, the deleted circle must reduce the value of the surrogate function $Q$ as much as possible (line 8).
This process is then repeated.
As a result, we shrink an upper bound of (\ref{Propose2}) and thereby hope to obtain a good approximation to the solution of (\ref{Propose1}).

\begin{algorithm}[t!]
  \caption{Gershgorin Circle-based Landmark Selection (GCLS)}
  \label{alg1}
  \begin{algorithmic}[1]
    \Require $\bm{\Phi}=[\phi_{ij}]\in\mathbb{R}^{N\times N}$, the number of landmarks $L$.
    \Ensure A set of landmarks $\mathcal{L}$.
    \State Initialize $\mathcal{L}=\emptyset$, $\bar{\mathcal{L}}=\{1,...,N\}$.
    \State Compute the bound of $\lambda(\bm{\Phi})$: $b_{\min}$, $b_{\max}$.
    \State Set $\alpha=\max\{0,-b_{\min}\}$, construct $\bm{\Psi}=\bm{\Phi}+\alpha\bm{I}_{N}$.
    \For{$i=1:N$}
           \State Initialize $c_i=\psi_{ii}$, $r_i=\sum_{j\neq i}|\psi_{ij}|$, $s_i=r_i$.
    \EndFor
    \For{$l=1:L$}
    		  \State Select landmark: $\hat{i}=\arg\min_{i\in\bar{\mathcal{L}}}Q(\bar{\mathcal{L}}\setminus i)$.
           \State $\mathcal{L}=\mathcal{L}\cup \hat{i}$, $\bar{\mathcal{L}}=\bar{\mathcal{L}}\setminus\hat{i}$.
           \For{$i\in \bar{\mathcal{L}}$}
        	         \State Update $s_i=s_i-\psi_{i\hat{i}}$.
           \EndFor
    \EndFor
  \end{algorithmic}
\end{algorithm}

\subsection{Further analysis and comparisons}\label{sec:analysis}
Compared with existing methods like MaxMinGeo, ApproxDPP, and MinCond, our GCLS method has several advantages. We summarize our comparison of these algorithms in Table~\ref{tab1}.

\begin{itemize}
\item \textbf{Universality with respect to alignment matrix.}
Both {MaxMinGeo} and ApproxDPP require the alignment matrix to be a graph Laplacian matrix because they need to compute the geodesic (or pairwise) distance between samples.
For {MinCond}, the eigenvalues of the alignment matrix must lie in the interval $(0,2)$ because the logarithm of the alignment matrix is calculated in each step.
A contribution of our GCLS method is adding a pre-processing step to construct a well-conditioned alignment matrix, and proving that it does not change the ultimate result of semi-supervised manifold learning.
With the help of this pre-processing step, our algorithm can be used with arbitrary alignment matrices.
\item \textbf{Complexity and scalability.}
All four of these methods select $L$ landmarks from $N$ samples.
The space complexity of these algorithms is $\mathcal{O}(N)$ for storing the sparse alignment matrix defined on a $K$-NN graph.
The time complexity of these algorithms differ substantially.
The time complexity of {MaxMinGeo} is $\mathcal{O}(LN\log(N))$ because the shortest path algorithm is applied to compute the geodesic distance between samples.
The time complexity of {ApproxDPP} is $\mathcal{O}(KL)$, where $K$ is the number of neighbors for the samples in the $K$-NN graph, because DPP only takes advantage of the pairwise distance between samples and their neighbors in each step.
The time complexity of {MinCond} is $\mathcal{O}(LN^3)$ because MinCond computes the logarithm of the alignment matrix in each step.
Finally, the time complexity of our GCLS algorithm is $\mathcal{O}(NL)$.
The time complexity of our algorithm is only higher than that of DPP, which has good scalability for real-time and large-scale applications.
\end{itemize}

It should be noted that although both take advantage of a relaxation based on Gershgorin circle theorem, the method we proposed in this paper is very different from our previous work in~\cite{xu2015active} in the following two aspects.
Firstly, their motivations and the corresponding objective functions are different.
The method in~\cite{xu2015active} only aims to minimize the condition number of the remaining principal submatrix of alignment matrix (i.e.,~(\ref{opt1})).
Our method, however, aims to minimize the upper bound of error in~(\ref{bound}), which considers both the condition number term and the reciprocal terms related to landmark coverage.
It thus unifies the principles of both geometric and algebraic methods.
Secondly, the landmarking algorithm proposed in this paper does not require time-consuming matrix logarithm operation --- the Gershgorin circles of the original alignment matrix are used to select landmarks directly.
Such a simple strategy greatly reduces the computational complexity of the algorithm.

\begin{table}[!t]
  \centering
  \caption{Comparisons for various algorithms.\label{tab1}}
  \begin{threeparttable}[c]
      \begin{tabular}{
        @{\hspace{2pt}}c@{\hspace{2pt}}|
        @{\hspace{2pt}}c@{\hspace{2pt}}
        @{\hspace{2pt}}c@{\hspace{2pt}}
        @{\hspace{2pt}}c@{\hspace{2pt}}
        }
        \hline\hline
        Algorithm  &$\bm{\Phi}$  &Obj. function  &Complexity\\
        \hline
        {MaxMinGeo}&Laplacian    &$\max\min d_{ij}$   &$\mathcal{O}(LN\log(N))$ \\
        {ApproxDPP}      &Laplacian    &approx. $\max\min d_{ij}$   &$\mathcal{O}(KL)$ \\
        {MinCond}  &$\lambda\in(0,2)$ &$\min\kappa(\bm{\Phi}_{\bar{\mathcal{L}}\bar{\mathcal{L}}})$ &$\mathcal{O}(LN^3)$\\
        {GCLS} &Arbitrary &$\min\frac{\kappa(\bm{\Psi}_{\bar{\mathcal{L}}\bar{\mathcal{L}}})}{\|\bm{\Psi}_{\bar{\mathcal{L}}\mathcal{L}}\|_1}+\frac{\kappa(\bm{\Psi}_{\bar{\mathcal{L}}\bar{\mathcal{L}}})}{\|\bm{\Psi}_{\bar{\mathcal{L}}\bar{\mathcal{L}}}\|_1}$    &$\mathcal{O}(NL)$\\
        \hline\hline
      \end{tabular}
  \end{threeparttable}
\end{table}

\subsection{Summary}
Our GCLS method can be viewed as a generalization of existing methods, which is robust to various alignment matrices and has relatively low complexity.
It should be noted that like existing methods, our heuristic method can only obtain a suboptimal landmarking solution.
However, experimental results in the following section will show that our method performs well in practical situations.

\begin{figure*}[h!]
\centering
\subfigure[Face data]{
\includegraphics[width=0.98\linewidth]{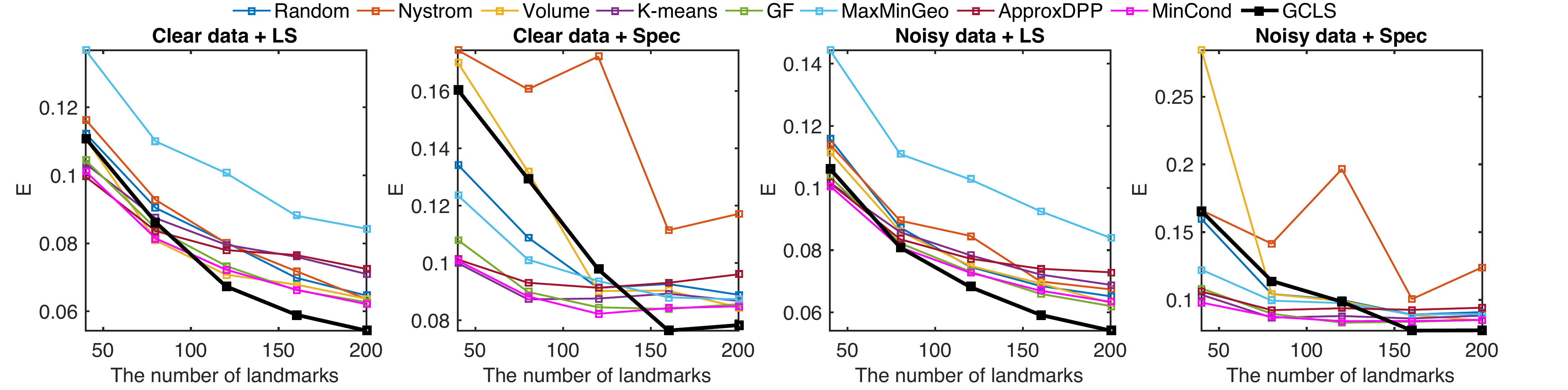}\label{fig_err_face}
}\\
\subfigure[Lips data]{
\includegraphics[width=0.98\linewidth]{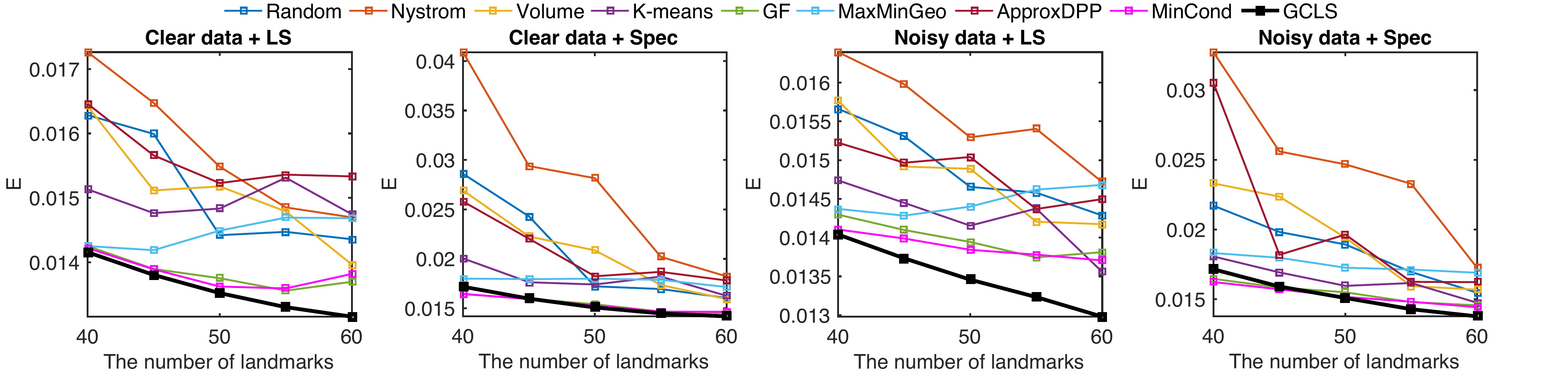}\label{fig_err_lips}
}\\
\subfigure[Jazz hand data]{
\includegraphics[width=0.98\linewidth]{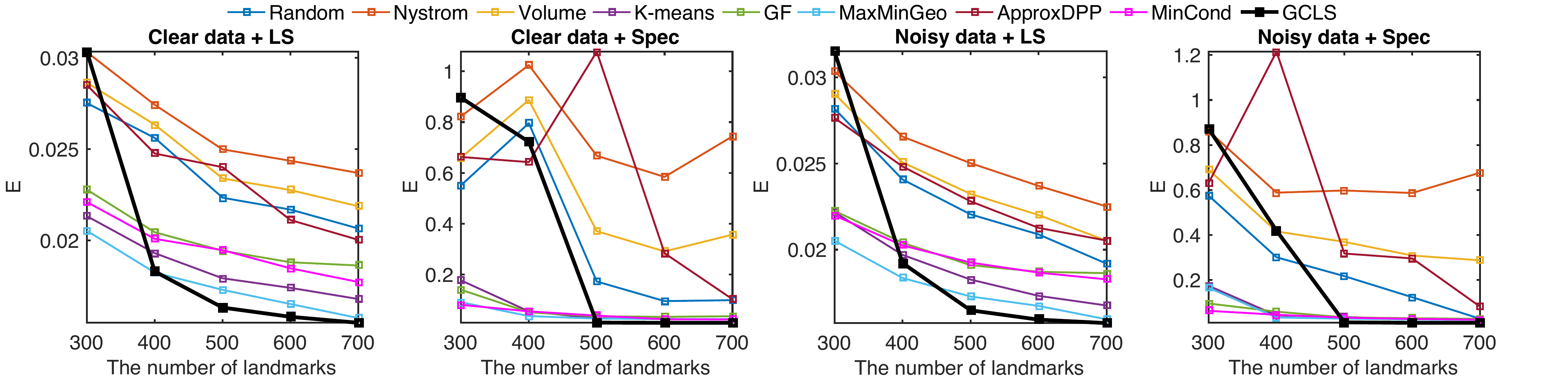}\label{fig_err_jazz}
}\\
\subfigure[Deformation grid data]{
\includegraphics[width=0.98\linewidth]{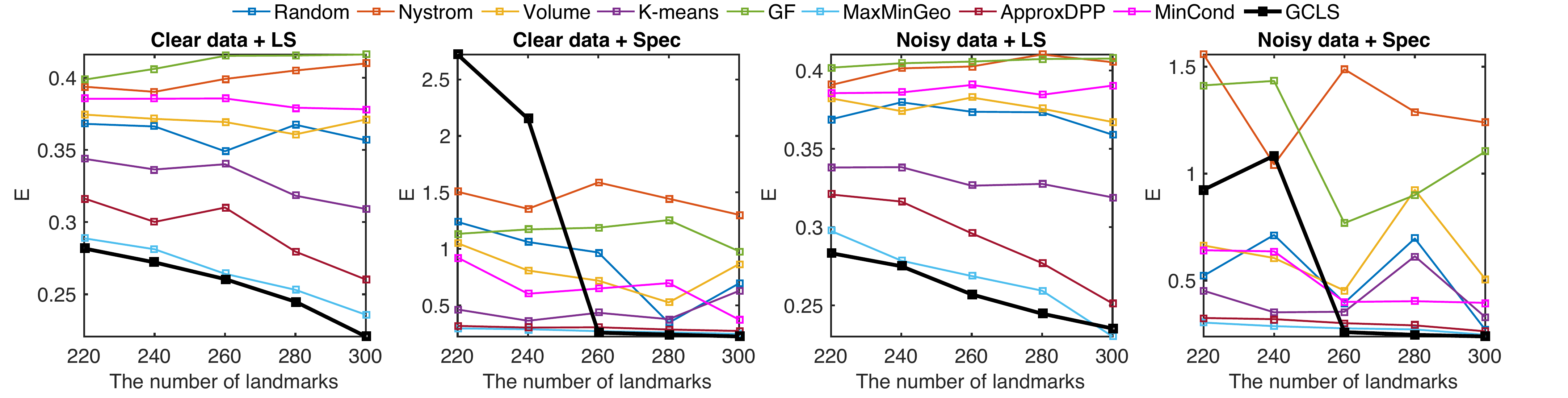}\label{fig_err_grid}
}
\caption{
The relative estimation errors obtained by our methods on various data sets.
In each row, the figures from left to right are: the estimation errors obtained by 1) the LS method on clean data; 2) the Spec method on clean data; 3) the LS method on noisy data; and 4) the Spec method on noisy data.}\label{fig1}
\end{figure*}

\section{Experiments}\label{sec:exp}
We compare our GCLS algorithm with existing landmarking algorithms on multiple data sets and in both regression and classification tasks.
Specifically, the competitors of our method include:
\begin{itemize}
\item \textbf{Random:} select landmarks uniformly without replacement.
\item \textbf{Nystr{\"o}m:} select landmarks non-uniformly with weight proportional to the $l_2$ norm of the column of $\bm{\Phi}$~\cite{drineas2006fast}.
\item \textbf{Volume:} the recent volume sampling method proposed in~\cite{derezinski2017unbiased}.
\item \textbf{K-means:} apply K-means as the pre-process of data, choose the samples that are close to clusters' centers as landmarks. This method is applied to the improved Nystr\"{o}m sampling method~\cite{zhang2008improved}.
\item \textbf{Gaussian field (GF) method:} the active learning method based on Gaussian field and harmonic functions~\cite{zhu2003combining}.
\item \textbf{MaxMinGeo:} select landmarks whose minimum geodesic distance are maximized~\cite{de2004sparse}.
\item \textbf{DPP:} select landmarks based on determinantal point processes~\cite{wachinger2015diverse}.
\item \textbf{MinCond:} select landmarks to minimize the condition number of the remaining alignment matrix~\cite{xu2015active}.
\end{itemize}
Among these methods, the MinCond and our GCLS algorithm are deterministic while the remaining methods are randomized.

In regression tasks, we test the algorithms with $20$ trials for each data set.
In each trial, we first select $N$ unlabeled samples randomly from a pool of samples and construct a $K$-NN graph.
Given the $N$ samples, we apply different landmarking algorithms to select $L$ landmarks and learn labels for the remaining samples via semi-supervised manifold learning methods.
For each landmarking algorithm and each $L$, we assume that the learning errors of different trials are drawn from a Gaussian distribution. 
Accordingly, we calculate the mean and the $95$\% confidence interval of the learning errors.
The relative learning error is calculated as
\begin{eqnarray*}
\begin{aligned}
E=\frac{\|\widehat{\bm{Z}}_u-\bm{Z}_u\|_F}{\|\bm{Z}_u\|_F}\times 100\%.
\end{aligned}
\end{eqnarray*}

For the purposes of landmarking, we use the graph Laplacian as the alignment matrix for each of the compared landmarking algorithms.
We evaluate our methods using both the {LS} algorithm in~\cite{yang2006semi} and the {Spec} algorithm in~\cite{zhang2008spectral} to achieve semi-supervised manifold learning.
In this step, we apply the original settings in these two references --- for the LS algorithm we use the graph Laplacian as the alignment matrix, and for the Spec algorithm we use the LTSA-based alignment matrix.
To investigate the robustness of each landmarking algorithm to noise, we test them on both the original clean data and on data corrupted by Gaussian noise with zero mean and variance $\sigma^2 = 0.01$.
Note that although the original GF method in~\cite{zhu2003combining} is designed for binary classification, it can be easily extend to regression task.
Additionally, the runtime of the algorithms and their bounds on the learning error of $\kappa(\bm{\Psi}_{\bar{\mathcal{L}}\bar{\mathcal{L}}})\left(\frac{1}{\|\bm{\Psi}_{\bar{\mathcal{L}}\mathcal{L}}\|_1}+\frac{1}{\|\bm{\Psi}_{\bar{\mathcal{L}}\bar{\mathcal{L}}}\|_1}\right)$ are also recorded.

In classification tasks, we test the algorithms over $20$ trials for each data set and record the averaged classification accuracy.
For each data set, we first construct an alignment matrix based on all the samples, and then select and label some samples with the help of different landmarking algorithms until each class contains $L$ labeled samples.
Finally, using these labeled samples, we train a classifier and test it on the remaining unlabeled samples.
A classifier based on the label-consistent K-SVD ({LCKSVD})~\cite{jiang2011learning} is applied.
The parameters of various landmark selection algorithms are set as follows.
For those randomized methods, the number of initial landmarks are $3$ random ones.
For {ApproxDPP}, we sample the point $\bm{x}_i$ as the new landmark with a probability $p_i\propto \prod_j(1-g_D(\bm{x}_i-\bm{x}_j))$, where $\bm{x}_j$ is the $j^{\text{th}}$ existing landmark and the function $g_D(\bm{x})=\exp(-{\|\bm{x}\|_2^2}/{2D^2})$.
The bandwidth $D$ is equal to the dimension of sample space.

\subsection{Regression tasks}
We test the different landmarking methods on four data sets, including the face data set from~\cite{tenenbaum2000global}, the lips and the Jazz hand data sets from~\cite{rahimi2005learning}, and the deformed grid data set from~\cite{tian2013hierarchical}.
For each data set, we select $N$ samples randomly and landmark $L$ of them in each trial.
\begin{itemize}
\item For the face data set, $N=500$ face images (with size $64\times 64$) are selected randomly, where the labels include a lighting intensity and two pose parameters ($d=3$).
We apply Principal Component Analysis (PCA) to reduce the dimension of data to $D=200$, and then, we construct a $K$-NN graph with $K=30$ based on the samples obtained by PCA.
For each landmark selection algorithm, we select $L=[40,...,200]$ landmarks respectively and then learn the labels for the remaining samples directly.
\item For the lips data set, $N=80$ lip images (with size $40\times 45$) are selected randomly, where the labels include four key points indicating the deformation of mouth ($d=8$).
We apply PCA to obtain the samples with $D=196$ and construct a $K$-NN graph with $K=35$.
For each landmark selection algorithm, we select $L=[40,...,60]$ landmarks respectively and then learn the labels for the remaining samples directly.
\item For the Jazz hand data set, $N=1000$ hand motion images (with size $66\times 88$, $D=5808$) are selected randomly, whose labels include four key points indicating the motion of two arms ($d=8$).
Similarly, we apply PCA to obtain the samples with $D=498$ and a $K$-NN graph with $K=30$.
For each landmark selection algorithm, we select $L=[300,...,700]$ landmarks respectively and then learn the labels for the remaining samples directly.
\item For the deformed grid data set, a reference grid image is given in Fig.~\ref{Fig4a}, whose key points are shown as red crosses.
Using various operations of deformation, a set of deformed images is generated.
We aim to label a small number of key points in these images and estimate the locations of the remaining key points.
Specifically, we segment each image into $25$ patches (with size $22\times 22$, $D=484$), each of which contains $4$ key points ($d=8$), as Fig.~\ref{Fig4a} shows.
In each trial, we randomly select $N=1500$ patches from the deformed images and then construct a $K$-NN graph with $K=30$.
For each landmark selection algorithm, we select $L=[220,...,300]$ landmarks and then learn labels for the remaining samples directly.
\end{itemize}

\begin{table}[t!]
  \centering
  \caption{The learning errors of various methods on clear data (\%).}\label{tab3}
  \begin{threeparttable}[c]
      \begin{tabular}{
        @{\hspace{1pt}}c@{\hspace{1pt}}|
        @{\hspace{2pt}}c@{\hspace{2pt}}
        @{\hspace{2pt}}c@{\hspace{2pt}}
        @{\hspace{2pt}}c@{\hspace{2pt}}
        @{\hspace{2pt}}c@{\hspace{2pt}}
        @{\hspace{2pt}}c@{\hspace{1pt}}
        }
        \hline\hline
        Data    &\multicolumn{5}{c}{Face $N$=500}\\ \hline
        $L$     &40  &80  &120 &160 &200\\ \hline
Random
&11.23$\pm$0.23 &9.04$\pm$0.36 &8.03$\pm$0.25 &6.99$\pm$0.17 &6.47$\pm$0.24\\
Nystr{\"o}m
&11.63$\pm$0.58 &9.27$\pm$0.24 &8.01$\pm$0.35 &7.18$\pm$0.11 &6.36$\pm$0.25\\
Volume
&11.08$\pm$0.37 &\textbf{8.08}$\pm$0.26 &7.08$\pm$0.14 &6.79$\pm$0.16 &6.37$\pm$0.10\\
K-means
&\textbf{10.00}$\pm$0.22 &8.73$\pm$0.08 &7.96$\pm$0.15 &7.60$\pm$0.23 &7.10$\pm${0.07}\\
GF
&10.44$\pm${0.19} &8.44$\pm$0.20 &7.34$\pm$0.14 &6.62$\pm$0.08 &6.27$\pm$0.09\\
MaxMinGeo
&12.36$\pm$0.44 &10.12$\pm${0.17} &9.37$\pm$0.15 &8.80$\pm${0.06} &8.43$\pm$0.16 \\
ApproxDPP
&9.96$\pm$0.24 &8.36$\pm$0.20 &7.80$\pm$0.17 &7.66$\pm$0.21 &7.25$\pm$0.16 \\
MinCond
&10.08$\pm$0.25 &8.14$\pm$0.23 &7.23$\pm${0.11} &6.64$\pm$0.10 &6.21$\pm$0.10 \\
GCLS
&11.09$\pm$0.30 &8.61$\pm$0.49 &\textbf{6.74}$\pm$0.20 &\textbf{5.90}$\pm$0.13 &\textbf{5.44}$\pm$0.10 \\
\hline\hline
        	Data     &\multicolumn{5}{c}{Lips $N$=80}\\ \hline
        	$L$      &40  &45  &50  &55  &60\\ \hline
Random
&1.63$\pm$0.07 &1.60$\pm$0.07 &1.44$\pm$0.05 &1.45$\pm$0.04 &1.44$\pm$0.05 \\
Nystr{\"o}m
&1.73$\pm$0.38 &1.65$\pm$0.21 &1.55$\pm$0.43 &1.49$\pm$0.09 &1.47$\pm$0.09 \\
Volume
&1.64$\pm$0.09 &1.51$\pm$0.04 &1.52$\pm$0.07 &1.48$\pm$0.11 &1.40$\pm$0.04 \\
K-means
&1.51$\pm$0.04 &1.48$\pm$0.03 &1.48$\pm$0.03 &1.53$\pm$0.04 &1.47$\pm$0.04 \\
GF
&1.43$\pm${0.02} &1.39$\pm$0.02 &1.38$\pm$0.02 &1.36$\pm$0.02 &1.37$\pm$0.02 \\
MaxMinGeo
&1.43$\pm$0.02 &1.42$\pm$0.02 &1.45$\pm$0.03 &1.47$\pm$0.04 &1.47$\pm$0.04 \\
ApproxDPP
&1.65$\pm$0.08 &1.57$\pm$0.06 &1.52$\pm$0.04 &1.54$\pm$0.04 &1.53$\pm$0.05 \\
MinCond
&\textbf{1.42}$\pm$0.02 &1.39$\pm${0.02} &1.36$\pm$0.02 &1.36$\pm$0.02 &1.38$\pm$0.02 \\
GCLS
&\textbf{1.42}$\pm$0.03 &\textbf{1.38}$\pm${0.02} &\textbf{1.35}$\pm${0.01} &\textbf{1.33}$\pm${0.01} &\textbf{1.32}$\pm${0.02} \\
\hline\hline        	
        	Data     &\multicolumn{5}{c}{Jazz Hand $N$=1000}\\ \hline
        	$L$      &300 &400 &500 &600 &700\\ \hline
Random
&2.75$\pm$0.07 &2.56$\pm$0.04 &2.23$\pm$0.04 &2.17$\pm$0.04 &2.07$\pm$0.08 \\
Nystr{\"o}m
&3.03$\pm$0.12 &2.74$\pm$0.11 &2.50$\pm$0.18 &2.44$\pm$0.11 &2.37$\pm$0.09 \\
Volume
&2.86$\pm$0.09 &2.63$\pm$0.07 &2.34$\pm$0.10 &2.28$\pm$0.07 &2.19$\pm$0.09 \\
K-means
&2.13$\pm$0.02 &1.93$\pm${0.01} &1.79$\pm$0.01 &1.74$\pm$0.01 &1.68$\pm$0.02 \\
GF
&2.28$\pm$0.02 &2.05$\pm$0.02 &1.94$\pm$0.03 &1.88$\pm$0.02 &1.86$\pm$0.03 \\
MaxMinGeo
&\textbf{2.05}$\pm${0.01} &\textbf{1.82}$\pm${0.01} &1.73$\pm$0.01 &1.65$\pm$0.01 &1.58$\pm$0.01 \\
ApproxDPP
&2.85$\pm$0.08 &2.48$\pm$0.07 &2.40$\pm$0.16 &2.11$\pm$0.04 &2.00$\pm$0.04 \\
MinCond
&2.21$\pm$0.01 &2.01$\pm$0.01 &1.95$\pm$0.01 &1.85$\pm$0.01 &1.77$\pm$0.02 \\
GCLS
&3.03$\pm$0.08 &1.83$\pm$0.06 &\textbf{1.05}$\pm${0.01} &\textbf{0.99}$\pm${0.01} &\textbf{0.98}$\pm${0.01} \\
\hline\hline       	
        	Data     &\multicolumn{5}{c}{Grid $N$=1500}\\ \hline
        	$L$      &220 &240 &260 &280 &300\\ \hline					
Random
&36.81$\pm$1.52 &36.63$\pm${0.63} &34.91$\pm$0.71 &34.64$\pm$1.16 &35.67$\pm$1.24 \\
Nystr{\"o}m
&39.38$\pm$0.67 &39.03$\pm$0.72 &39.92$\pm$0.91 &40.50$\pm$0.89 &40.99$\pm$1.50 \\
Volume
&37.45$\pm$1.34 &37.16$\pm$0.82 &36.94$\pm$0.59 &36.08$\pm$1.69 &37.11$\pm$1.35 \\
K-means
&34.39$\pm$0.59 &33.65$\pm$0.88 &34.01$\pm$1.36 &31.83$\pm$1.09 &30.91$\pm$1.63 \\
GF
&39.87$\pm$1.17 &40.62$\pm$1.22 &41.52$\pm$1.35 &41.54$\pm$1.18 &41.61$\pm$1.38 \\
MaxMinGeo
&28.88$\pm$0.73 &28.12$\pm$1.02 &26.42$\pm$0.78 &25.31$\pm$1.45 &23.59$\pm$1.80 \\
ApproxDPP
&31.64$\pm$0.96 &30.02$\pm$0.89 &30.64$\pm$1.08 &27.96$\pm$1.01 &26.02$\pm$1.16 \\
MinCond
&38.55$\pm$1.11 &38.55$\pm$1.19 &38.57$\pm$1.20 &37.93$\pm$1.44 &37.40$\pm$1.31 \\
GCLS
&\textbf{28.17}$\pm${0.57} &\textbf{27.21}$\pm$0.71 &\textbf{25.89}$\pm${0.66} &\textbf{24.20}$\pm${0.64} &\textbf{22.08}$\pm${0.92} \\
        \hline\hline
      \end{tabular}
  \end{threeparttable}
\end{table}

\begin{table}[t!]
  \centering
  \caption{The learning errors of various methods on noisy data (\%).}\label{tab3_2}
  \begin{threeparttable}[c]
      \begin{tabular}{
        @{\hspace{1pt}}c@{\hspace{1pt}}|
        @{\hspace{2pt}}c@{\hspace{2pt}}
        @{\hspace{2pt}}c@{\hspace{2pt}}
        @{\hspace{2pt}}c@{\hspace{2pt}}
        @{\hspace{2pt}}c@{\hspace{2pt}}
        @{\hspace{2pt}}c@{\hspace{1pt}}
        }
        \hline\hline
        Data    &\multicolumn{5}{c}{Face $N$=500}\\ \hline
        $L$     &40  &80  &120 &160 &200\\ \hline
Random
&11.58$\pm$0.70 &8.73$\pm$0.35 &7.46$\pm$0.27 &6.85$\pm$0.08 &6.52$\pm$0.14 \\
Nystr{\"o}m
&11.40$\pm$0.61 &8.95$\pm$0.50 &8.46$\pm$0.19 &6.99$\pm$0.17 &6.74$\pm$0.32 \\
Volume
&11.15$\pm$0.62 &8.61$\pm$0.37 &7.49$\pm$0.31 &6.93$\pm$0.17 &6.30$\pm$0.12 \\
K-means
&10.17$\pm${0.09} &8.58$\pm$0.20 &7.84$\pm$0.16 &7.21$\pm$0.18 &6.88$\pm$0.18 \\
GF
&10.36$\pm$0.28 &8.23$\pm$0.15 &7.31$\pm$0.11 &6.59$\pm${0.06} &6.20$\pm$0.07 \\
MaxMinGeo
&12.23$\pm$0.46 &9.95$\pm$0.48 &9.76$\pm$0.12 &8.93$\pm$0.10 &8.39$\pm$0.23 \\
ApproxDPP
&10.16$\pm$0.14 &8.35$\pm$0.11 &7.73$\pm$0.15 &7.40$\pm$0.23 &7.29$\pm$0.18 \\
MinCond
&\textbf{9.80}$\pm$0.25 &8.09$\pm${0.24} &7.27$\pm${0.10} &6.69$\pm$0.07 &6.34$\pm${0.06} \\
GCLS
&10.62$\pm$0.25 &\textbf{8.08}$\pm$0.43 &\textbf{6.81}$\pm$0.25 &\textbf{5.92}$\pm$0.11 &\textbf{5.41}$\pm$0.09 \\
\hline\hline
        	Data     &\multicolumn{5}{c}{Lips $N$=80}\\ \hline
        	$L$      &40  &45  &50  &55  &60\\ \hline
Random
&1.57$\pm$0.06 &1.53$\pm$0.10 &1.47$\pm$0.04 &1.46$\pm$0.04 &1.43$\pm$0.04 \\
Nystr{\"o}m
&1.64$\pm$0.08 &1.60$\pm$0.14 &1.53$\pm$0.17 &1.54$\pm$0.45 &1.47$\pm$0.16 \\
Volume
&1.58$\pm$0.06 &1.49$\pm$0.04 &1.49$\pm$0.05 &1.42$\pm$0.05 &1.42$\pm$0.06 \\
K-means
&1.47$\pm$0.02 &1.44$\pm$0.03 &1.42$\pm$0.03 &1.44$\pm$0.04 &1.36$\pm$0.03 \\
GF
&1.43$\pm${0.01} &1.41$\pm$0.02 &1.39$\pm$0.02 &1.37$\pm$0.02 &1.38$\pm$0.02 \\
MaxMinGeo
&1.44$\pm$0.03 &1.43$\pm$0.03 &1.44$\pm$0.02 &1.46$\pm$0.03 &1.47$\pm$0.04 \\
ApproxDPP
&1.52$\pm$0.05 &1.50$\pm$0.02 &1.50$\pm$0.09 &1.44$\pm$0.04 &1.45$\pm$0.05 \\
MinCond
&1.41$\pm$0.03 &1.40$\pm$0.02 &1.38$\pm$0.02 &1.38$\pm$0.02 &1.37$\pm$0.02 \\
GCLS
&\textbf{1.40}$\pm$0.03 &\textbf{1.37}$\pm${0.02} &\textbf{1.35}$\pm${0.01} &\textbf{1.32}$\pm${0.02} &\textbf{1.30}$\pm${0.02} \\
\hline\hline        	
        	Data     &\multicolumn{5}{c}{Jazz Hand $N$=1000}\\ \hline
        	$L$      &300 &400 &500 &600 &700\\ \hline
Random
&2.82$\pm$0.05 &2.41$\pm$0.05 &2.20$\pm$0.05 &2.09$\pm$0.01 &1.92$\pm$0.01 \\
Nystr{\"o}m
&3.03$\pm$0.07 &2.66$\pm$0.06 &2.50$\pm$0.17 &2.37$\pm$0.03 &2.25$\pm$0.12 \\
Volume
&2.90$\pm$0.06 &2.51$\pm$0.05 &2.32$\pm$0.10 &2.20$\pm$0.02 &2.05$\pm$0.05 \\
K-means
&2.22$\pm$0.02 &1.97$\pm$0.01 &1.83$\pm${0.01} &1.73$\pm$0.01 &1.68$\pm${0.01} \\
GF
&2.23$\pm$0.03 &2.04$\pm$0.02 &1.91$\pm$0.03 &1.87$\pm$0.01 &1.86$\pm$0.01 \\
MaxMinGeo
&\textbf{2.05}$\pm${0.01} &\textbf{1.84}$\pm${0.01} &1.73$\pm${0.01} &1.67$\pm${0.01} &1.60$\pm$0.02 \\
ApproxDPP
&2.76$\pm$0.05 &2.48$\pm$0.04 &2.28$\pm$0.04 &2.13$\pm$0.02 &2.05$\pm$0.03 \\
MinCond
&2.20$\pm$0.03 &2.02$\pm$0.01 &1.93$\pm$0.01 &1.87$\pm$0.01 &1.83$\pm$0.01 \\
GCLS
&3.15$\pm$0.16 &1.92$\pm$0.07 &\textbf{1.19}$\pm$0.02 &\textbf{0.99}$\pm$0.01 &\textbf{0.99}$\pm$0.01 \\
\hline\hline       	
        	Data     &\multicolumn{5}{c}{Grid $N$=1500}\\ \hline
        	$L$      &220 &240 &260 &280 &300\\ \hline					
Random
&36.88$\pm$0.97 &37.97$\pm$0.70 &37.36$\pm$0.72 &37.33$\pm$0.97 &27.03$\pm$0.93 \\
Nystr{\"o}m
&39.07$\pm$0.68 &40.14$\pm$0.89 &40.26$\pm${0.62} &41.02$\pm$0.90 &40.53$\pm$0.70 \\
Volume
&38.22$\pm${0.64} &37.41$\pm$0.96 &38.27$\pm$0.90 &37.55$\pm${0.82} &36.71$\pm$0.91 \\
K-means
&33.80$\pm$0.86 &33.83$\pm$0.73 &32.64$\pm$1.45 &32.76$\pm$1.22 &31.89$\pm$1.65 \\
GF
&40.18$\pm$1.14 &40.46$\pm$1.12 &40.57$\pm$1.09 &40.73$\pm$1.29 &40.76$\pm$1.38 \\
MaxMinGeo
&29.76$\pm$0.71 &27.85$\pm$0.64 &26.87$\pm$0.64 &25.93$\pm$1.13 &\textbf{23.01}$\pm$1.52 \\
ApproxDPP
&32.08$\pm$0.92 &31.63$\pm$1.11 &29.60$\pm$1.34 &27.69$\pm$1.60 &25.12$\pm$1.47 \\
MinCond
&38.55$\pm$1.16 &38.60$\pm$1.29 &39.08$\pm$1.30 &38.45$\pm$1.30 &39.04$\pm$1.39 \\
GCLS
&\textbf{28.33}$\pm$0.78 &\textbf{27.50}$\pm$0.63 &\textbf{25.70}$\pm$0.87 &\textbf{24.46}$\pm$1.20 &23.52$\pm$1.14 \\
        \hline\hline
      \end{tabular}
  \end{threeparttable}
\end{table}

\begin{figure*}[!t]
\centering
\subfigure[Some estimation results of lights and poses]{
\includegraphics[width=0.95\linewidth]{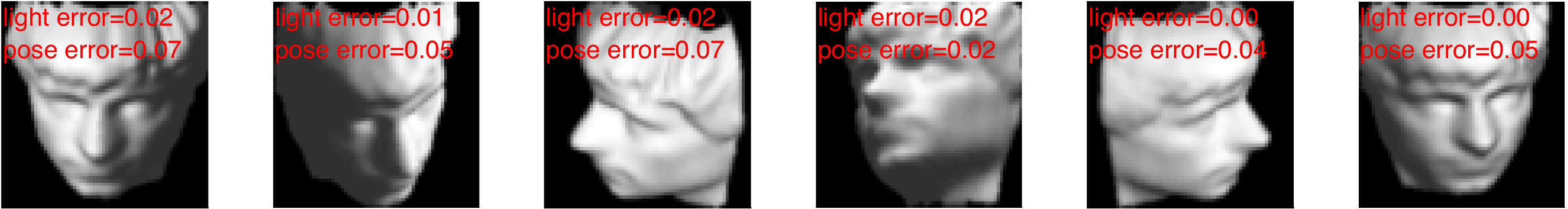}\label{Fig_face}
}
\subfigure[Some tracking results of key points of lips]{
\includegraphics[width=0.95\linewidth]{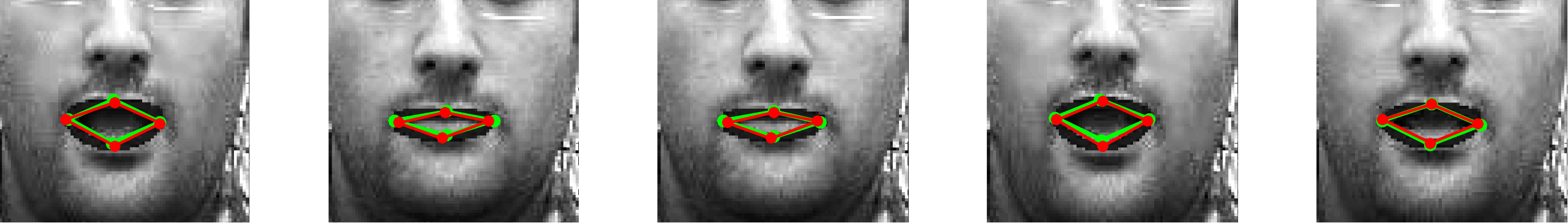}\label{Fig_lips}
}
\subfigure[Some tracking results of key points of arms]{
\includegraphics[width=0.95\linewidth]{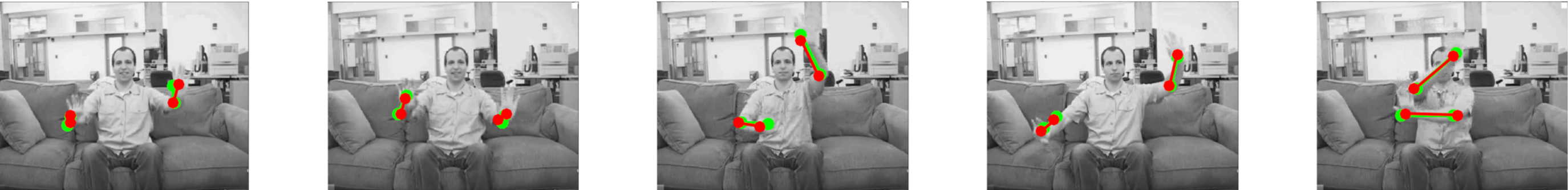}\label{Fig_jazz}
}
\subfigure[Reference and some tracking results of key points of deformed grids]{
\includegraphics[width=0.18\linewidth]{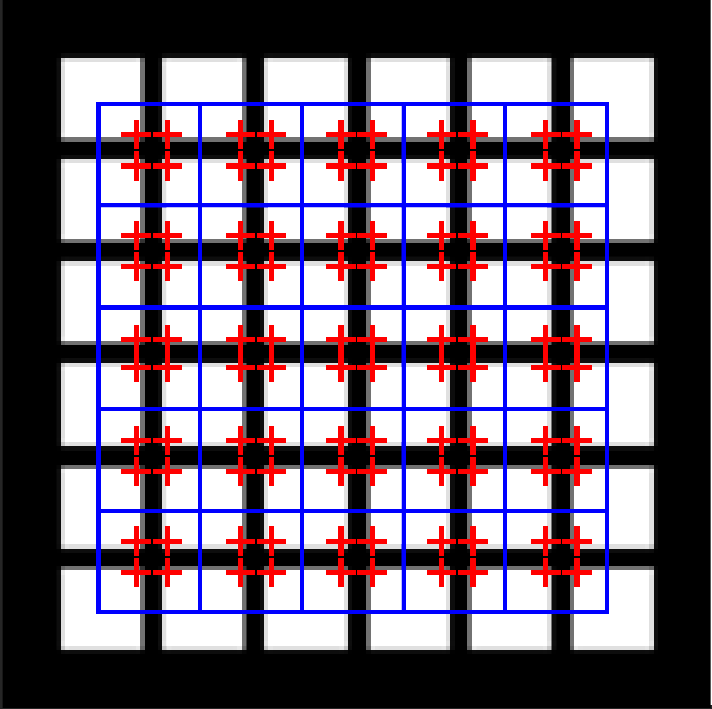}~~\label{Fig4a}
\includegraphics[width=0.18\linewidth]{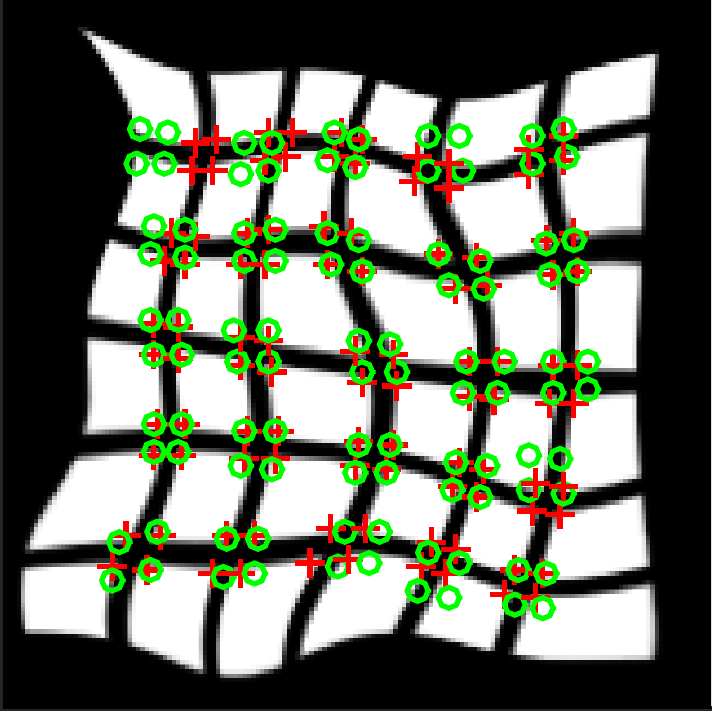}~~\label{Fig4b}
\includegraphics[width=0.18\linewidth]{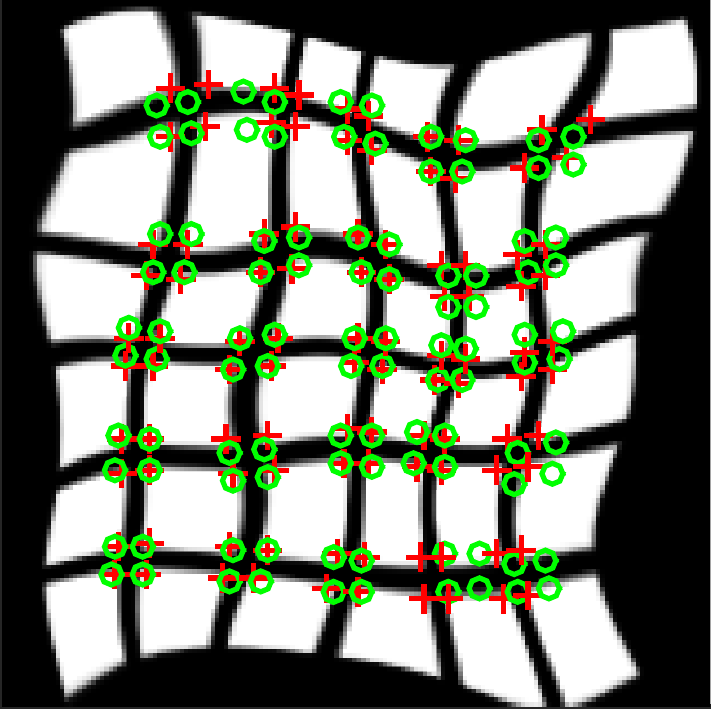}~~\label{Fig4c}
\includegraphics[width=0.18\linewidth]{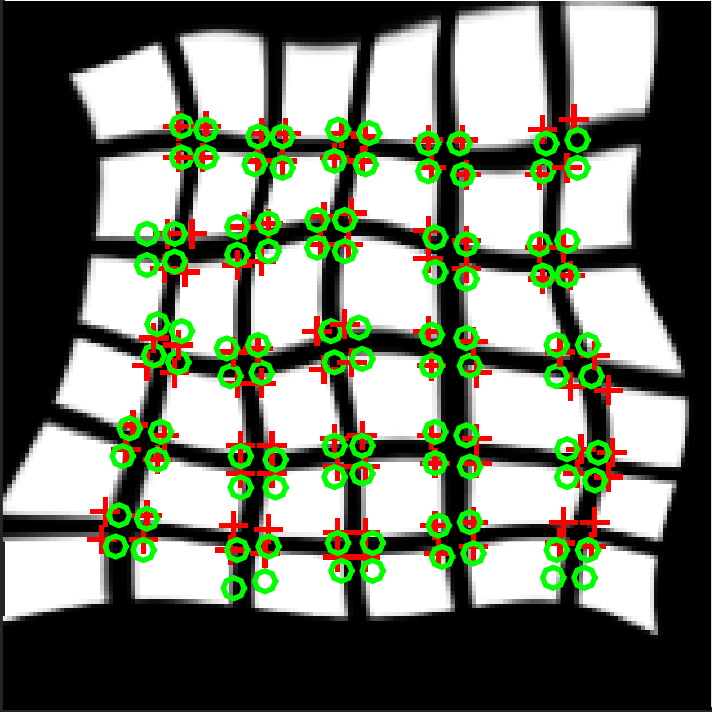}~~\label{Fig4d}
\includegraphics[width=0.18\linewidth]{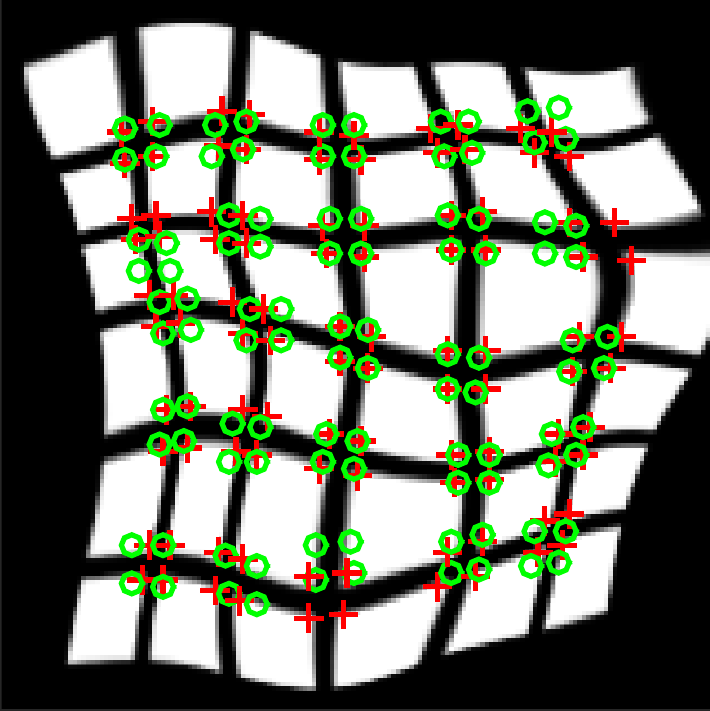}~~\label{Fig4e}
}
\caption{
The regression results of our method on various data sets.
In subfigure (a), the relative estimation errors are labeled as red.
In subfigures (b,c), the real and estimated key points are labeled as green and red, respectively.
In subfigure (d), the segmentation of reference grid (blue lines) and key points (red crosses).
In each deformed image, estimated locations of key points (green circles) and ground truth (red crosses) are given.
}\label{fig4}
\end{figure*}

\begin{figure*}[!h]
\centering
\subfigure[Runtime]{\includegraphics[height=0.25\linewidth]{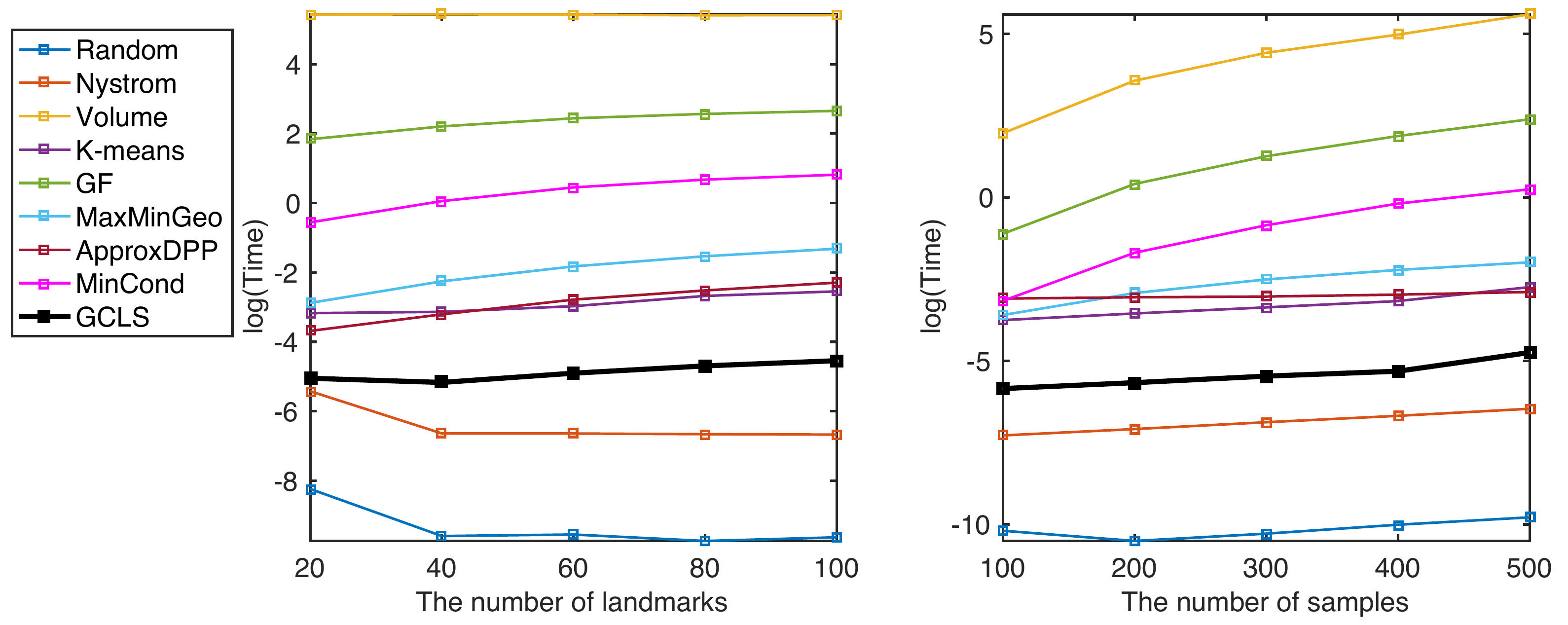}\label{figTime}
}
\subfigure[Error bound]{
\includegraphics[height=0.25\linewidth]{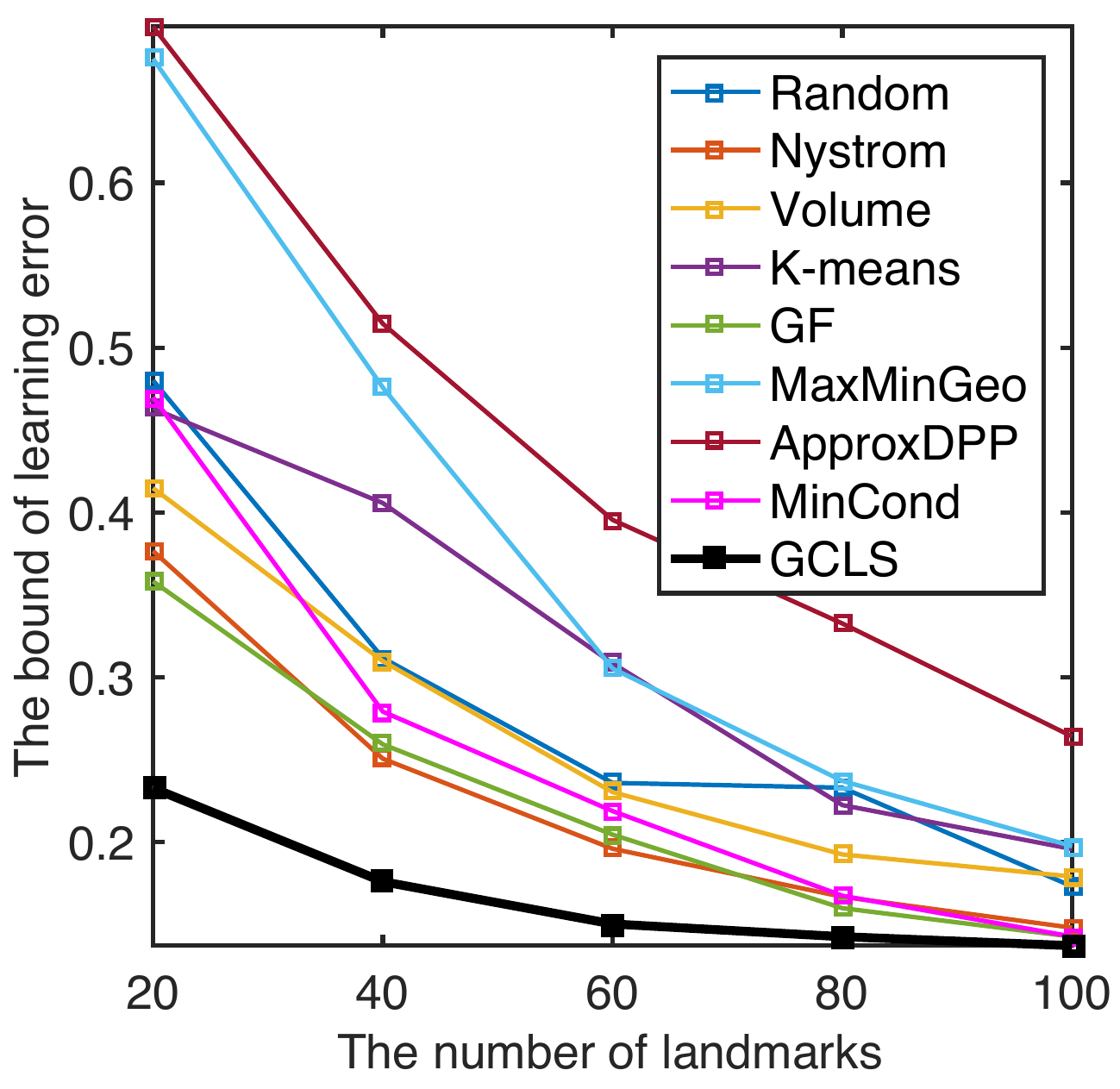}\label{fig_bound}
}
\caption{
(a) The run time versus the number of landmarks $L$ ($N=500$) and the number of samples $N$ ($L=50$) respectively.
(b) Comparison of the value of the objective function in~(\ref{Propose1}).
}\label{fig3}
\end{figure*}

The visual comparisons for various methods is shown in Fig.~\ref{fig1}.
According to Fig.~\ref{fig1}, we find that for all the data sets considered, selecting the graph Laplacian matrix as the alignment matrix in the landmarking step helps us to achieve improved learning results.
One possible reason is that the LTSA-based alignment matrix involves many hyperparameters~\cite{zhang2008spectral} which are difficult to set optimally for each specific data set.

The numerical comparisons for various landmarking methods on both clean and noisy data are shown in Tables~\ref{tab3} and \ref{tab3_2}, respectively.
We provide the semi-supervised learning results obtained by the LS algorithm.
In most situations, our GCLS algorithm achieves lower average learning error than its competitors with respect to various choices of $L$ on both clean and noisy data. 
Especially for the lips and the deformed grid data, our GCLS is consistently superior to its competitors when testing on both clean and noisy data.
For the face and the Jazz hand data, while our GCLS method performs worse than others when $L$ is small, its error reduces quickly when $L$ increases.
In general, for modestly large values of $L$ we observe that our GCLS method either clearly outperforms the others (in that the upper bound of the confidence interval for GCLS is lower than the lower bounds of all other methods' intervals) or is roughly comparable (with substantially overlapping confidence intervals).
In summary, when just landmarking a few samples, our method is at least comparable to other methods.
As the number of landmarks increases, the superiority of our method becomes more and more pronounced.

The runtime of each algorithm is given in Fig.~\ref{figTime}.
Consistent with our analysis, the speed of our algorithm is competitive and achieves a balance between complexity and performance.
In particular, volume sampling is the most time-consuming method, which requires computing the inverse of a $D\times D$ matrix $\mathcal{O}(N)$ times in each iteration.
The GF method also need to calculate matrix inverses, resulting in a runtime only slightly shorter than volume sampling.
Although the results of {MinCond} are close to (though still worse than) ours, its run time is about $100$ times longer.
Only the ApproxDPP, the column norm-based Nystr{\"o}m sampling, and the random sampling method have lower complexity than our method, but our method obtains significantly lower error than them.
Finally, in Fig.\ref{fig_bound} we compare the error bound which serves as the objective function in (\ref{Propose1}) versus $L$ for the different algorithms.
As expected, GCLS obtains a much tighter bound than others.

\subsection{Classification tasks}
We also apply our landmarking algorithm to the classification tasks of the AR-face and the Extended YaleB data sets, respectively.
We use a subset of the AR-face data set~\cite{martinez1998ar} consisting of $2600$ images from $100$ subjects.
The Extended YaleB contains $2414$ frontal face images of $38$ persons.
After selecting labeled samples via the different landmarking algorithms, we apply the {LCKSVD} algorithm in~\cite{jiang2011learning} to a learn dictionary and attain sparse codes for the samples.
The size of the dictionary and the parameters of the learning algorithm follow the setting in~\cite{jiang2011learning}.
Finally, taking the learned sparse codes as the features of the samples, we train an SVM classifier~\cite{cortes1995support}.
Table~\ref{tab4} compares classification results corresponding to various algorithms.
We observe that GCLS significantly improves classification accuracy, especially in the case of very few landmarks.

\begin{table}[t!]
  \centering
  \caption{Classification accuracy (\%).}\label{tab4}
  \begin{threeparttable}[c]
      \begin{tabular}{
        @{\hspace{2pt}}c@{\hspace{2pt}}|
        @{\hspace{2pt}}c@{\hspace{2pt}}
        @{\hspace{2pt}}c@{\hspace{2pt}}
        @{\hspace{2pt}}c@{\hspace{2pt}}
        @{\hspace{2pt}}|c@{\hspace{2pt}}
        @{\hspace{2pt}}c@{\hspace{2pt}}
        @{\hspace{2pt}}c@{\hspace{2pt}}
        }
        \hline\hline
        	Data &\multicolumn{3}{c}{AR Face} &\multicolumn{3}{|c}{Extended YaleB}\\ \hline
        $L$  &10/Class  &15/Class  &20/Class  &15/Class  &20/Class &25/Class\\ \hline
        Random      &82.06 &85.26 &87.67 &74.62 &79.05 &93.16\\
        Nystr{\"o}m &83.53 &86.67 &89.17 &74.46 &83.50 &93.22\\
        Volume      &85.50 &87.00 &89.35 &75.96 &84.22 &93.91\\
        GF          &86.65 &87.17 &\textbf{90.17} &74.46 &84.38 &94.21\\
        K-means     &82.87 &85.89 &88.25 &74.76 &83.23 &93.20\\
        MaxMinGeo   &85.21 &87.33 &89.17& 75.30 &83.73 &93.22\\
        ApproxDPP   &82.65 &84.54 &87.27& 72.98 &81.02 &89.19\\
        MinCond     &\textbf{86.67}&87.00 & 89.00& 74.46 & 83.89 & 93.91\\
        GCLS        &\textbf{86.67}&\textbf{88.67} &89.67 &\textbf{75.96} &\textbf{84.97} &\textbf{94.22}\\
        \hline\hline
      \end{tabular}
  \end{threeparttable}
\end{table}

\section{Conclusion}\label{sec:con}
In this paper, we study the active manifold learning problem and propose a landmark selection algorithm based on the Gershgorin circle theorem.
We establish connections among various landmark selection algorithms and propose a unified algorithmic framework.
Essentially, we treat the manifold landmarking problem as a combinatorial optimization problem, and the proposed landmark selection algorithm provides a heuristic solution.
Compared with other competitors, our GCLS algorithm has lower complexity and higher performance in both regression and classification tasks.
Although our method empirically achieves encouraging performance in various learning tasks, we do not provide any global optimality guarantees, and it is theoretically possible that it may output unsatisfying local optimal solutions.
In the future, we hope to provide a more rigorous theoretical underpinning for this algorithm and also plan to apply more sophisticate learning algorithms, e.g., genetic algorithms~\cite{deb2002fast} and neural network-based methods~\cite{smith1999neural,khalil2017learning,kool2018attention}, to solve (\ref{Propose1}) or closely-related variants.


\appendices
\section*{Appendix}
\subsection{A review of manifold learning}
Let $\bm{X}=[\bm{x}_1,...,\bm{x}_N]\in\mathbb{R}^{D\times N}$ be a set of samples from a manifold $\mathcal{X}$.
Let $\bm{y}_i\in \mathbb{R}^{d}$ ($d\ll D$) represent the unknown low-dimensional parameter vector corresponding to $\bm{x}_i$.
We can build a graph for the samples as follows: for each sample $\bm{x}_i$, its $K$ nearest neighbors are selected and denoted as $\bm{X}_i\in \mathbb{R}^{D\times K}$.
This graph provides us with a significant amount of geometrical information about the manifold.
We can then solve the manifold learning problem via several different strategies.

\begin{itemize}
\item \textbf{LLE} characterizes the local geometry of the manifold by finding linear coefficients that reconstruct each sample from its neighbors~\cite{roweis2000nonlinear}.
The coefficients can be learned by minimizing the reconstruction error.
The coefficients preserve the relationships among the samples and their neighbors, which are assumed to be inherited by the low-dimensional parameters.
Formally, we first compute coefficients by
\begin{eqnarray}\label{lleW}
\begin{aligned}
\min_{\bm{W}}&~\sideset{}{_{i=1}^{N}}\sum\|\bm{x}_i-\bm{X}_i\bm{w}_i\|_{2}^{2}\\
\text{s.t.}&~\sideset{}{_{j=1}^{K}}\sum w_{ij}=1,
\end{aligned}
\end{eqnarray}
where $\bm{W}=[\bm{w}_1,..,\bm{w}_N]$, $\bm{w}_i=[w_{i1},..,w_{iK}]^T$.
Then, $\bm{Y}=[\bm{y}_1,...,\bm{y}_N]$ is computed by
\begin{eqnarray}\label{lle}
\begin{aligned}
\min_{\bm{Y}}\sideset{}{_{i=1}^{N}}\sum\|\bm{y}_i-\bm{Y}_i\bm{w}_i\|_{2}^{2}
=\min_{\bm{Y}}\mbox{tr}(\bm{Y\Phi Y}^{T}),
\end{aligned}
\end{eqnarray}
where $\bm{Y}_i=[\bm{y}_{i1},...,\bm{y}_{iK}]$ contains the neighbors of $\bm{y}_i$, whose columns are the $K$ nearest samples of $\bm{y}_i$.
The entries of $\bm{\Phi}$ are given by $\phi_{ij}=\delta_{ij}-w_{ij}-w_{ji}+\bm{w}_{i}^{T}\bm{w}_j$, $i,j=1,..,N$, where $\delta_{ij}$ is $1$ if $i=j$ and $0$ otherwise.
We add constraints $\sum_{i}\bm{y}_i=\bm{0}$ and $\bm{YY}^T=\bm{I}_d$ to fix the scaling, translation, and rotation of the latent variables.
The resulting problem reduces to finding the set of eigenvectors corresponding to the $2^{\text{nd}}$ to $(d+1)^{\text{th}}$ smallest eigenvalues of $\bm{\Phi}$.
\item \textbf{LTSA} also tries to capture the local geometry of $\mathcal{X}$~\cite{zhang2004principal}.
Assume that there exists a mapping from the latent space to the ambient space, i.e., $f:~\mathcal{Y}\mapsto\mathcal{X}$.
Instead of directly computing reconstruction coefficients, LTSA approximates the tangent space at each sample.
According to the Taylor expansion, for each $\bm{x}_j\in \bm{X}_i$ we have
\begin{eqnarray}\label{taylor}
\begin{aligned}
\bm{x}_i-\bm{x}_j \approx \bm{J}_i(\bm{y}_i-\bm{y}_j) = \bm{J}_i\bm{\theta}_j^{i}.
\end{aligned}
\end{eqnarray}
Here $\bm{J}_i=[\frac{\partial f}{\partial \bm{y}_i}]\in \mathbb{R}^{D\times d}$ is the Jacobian matrix of $f$ at $\bm{y}_i$, which can be calculated as the singular vectors corresponding to the largest $d$ singular values of $\bm{X}_i-\bm{x}_i\bm{e}^T$, $\bm{e}=[1,..,1]^T$.
Then $\bm{\Theta}_i=[\bm{\theta}_1^{i},...,\bm{\theta}_K^{i}]\in \mathbb{R}^{d\times K}$ are local coordinates of $\bm{X}_i$ in the tangent space.
After computing the local tangent space at each sample, the global coordinates $\bm{Y}$ are computed by aligning the local tangent spaces together.
Assuming that the corresponding global parameter vectors $\bm{Y}_i$ differ from the local ones $\bm{\Theta}_i$ by a local affine transformation, we minimize the errors of the transformation by
$\min_{\bm{\Phi}_i}\|\bm{Y}_i\bm{\Phi}_i\|_F^2$, where $\bm{\Phi}_i$ is the orthogonal projection whose null space is spanned by the columns of $[\bm{1},\bm{\Theta}_i]$.
In practice, we obtain $\bm{Y}$ by
\begin{eqnarray}\label{ltsa}
\begin{aligned}
\min_{\bm{Y}}\sideset{}{_{i=1}^{N}}\sum\|\bm{Y}_i\bm{\Phi}_i\|_{F}^2=
\min_{\bm{Y}}\mbox{tr}(\bm{Y}\bm{\Phi}\bm{Y}^T).
\end{aligned}
\end{eqnarray}
Here $\bm{\Phi}=\sum_{i=1}^{N}\bm{S}_i\bm{\Phi}_i\bm{S}_i^T$, where the $\bm{S}_i$ are 0-1 selection matrices ensuring $\bm{Y}_i=\bm{Y}\bm{S}_i$.
Adding the normalization conditions $\bm{YY}^T = \bm{I}_d$ and $\bm{Y}\bm{1} =\bm{0}$, the solution of (\ref{ltsa}) is the set of eigenvectors corresponding to the $2^{\text{nd}}$ to $(d+1)^{\text{th}}$ smallest eigenvalues of $\bm{\Phi}$.
In recent years, LTSA has been extended to more complicated parametric models~\cite{dollar2006learning}, achieving encouraging results in many applications~\cite{xu2013manifold,xu2014manifold}.
\item \textbf{ISOMAP} seeks an embedding preserving the geodesic distance between samples~\cite{tenenbaum2000global}.
The geodesic distances are computed by finding the shortest paths in the graph connecting neighboring data points.
Let $\bm{D}$ denote the matrix of squared geodesic distances.
Let $\bm{P}\in \mathbb{R}^{N\times N}$ denote the projection matrix $\bm{I}_N-\frac{1}{N}\bm{ee}^T$.
The low dimensional global coordinates are computed by finding the eigenvectors corresponding to the $d$ maximum eigenvalues of $\bm{A}=-\frac{1}{2}\bm{P}^T\bm{DP}$.
According to \cite{yang2006semi}, the ISOMAP problem can also be rewritten as follows:
Let $\bm{Q\Lambda Q}^T$ be the eigen-decomposition of $\bm{A}$. $\bm{Q}=[\bm{q}_1,..,\bm{q}_N]$ and $\bm{\Lambda}=diag(\lambda_1,..,\lambda_N)$, $\lambda_1\geq\lambda_2\geq..\geq\lambda_N$.
Then we compute the alignment matrix as
\begin{eqnarray}
\begin{aligned}
\bm{\Phi}=\lambda_1\bm{I}_N-\bm{A}
-\sum_{i=2}^{d}(\lambda_1-\lambda_i)\bm{q}_i\bm{q}_{i}^{T}
-\frac{\lambda_1}{N}\bm{ee}^T,
\end{aligned}
\end{eqnarray}
where $\bm{\Phi}$ has $d+1$ zero eigenvalues and its null space is spanned by $[\bm{q}_1,..,\bm{q}_d,\bm{e}]$. Therefore we can solve the ISOMAP problem via $\min_{\bm{Y}}\mbox{tr}(\bm{Y}\bm{\Phi}\bm{Y}^T)$ as well.
\end{itemize}

As shown above,
although the three algorithms compute the alignment matrix $\bm{\Phi}$ from different points of view, they all reduce to the same eigen-problem.

\subsection{Theoretical bound of condition number of submatrix}
Mathematically, given an $N \times N$ matrix $\bm{B}$ with singular values
$\sigma_1(\bm{B})\geq\sigma_2(\bm{B}) \geq \cdots\geq \sigma_N(\bm{B}) \geq 0$, if there is a gap between $\sigma_n(\bm{B})$ and $\sigma_{n+1}(\bm{B})$, and $\sigma_{n+1}(\bm{B})$ is sufficiently small, one may assume that $\bm{B}$ has a numerical rank $n$.
In this case, RRQR-factorization attempts to find a permutation matrix $\bm{\Pi}$ such that the QR factorization
\begin{eqnarray}\label{eq:B=RRQR}
\begin{aligned}
&\bm{B\Pi}=\bm{QR},\\
&\bm{R}=\left[\begin{array}{cc}
                 \bm{R}_{11} & \bm{R}_{12} \\
                 \bm{0}      & \bm{R}_{22} \end{array}\right],
\end{aligned}
\end{eqnarray}
satisfies that $\bm{R}_{11}\in\mathbb{R}^{n\times n}$ and $\bm{R}_{11}$'s smallest singular value
$\sigma_{\min}(\bm{R}_{11})\approx \sigma_n(\bm{B})$ and $\bm{R}_{22}$'s largest singular value $\sigma_{\max}(\bm{R}_{22})\approx\sigma_{n+1}(\bm{B})$, where
$\bm{Q}$ is orthogonal and $\bm{R}$ is upper triangular.
In essence, $\bm{R}_{11}$ captures the well-conditioned part of $\bm{B}$.
Readers can refer to~\cite{hong1992rank} for the details of RRQR.

An important property of RRQR is that there exists an RRQR such that
\begin{eqnarray}\label{ineq:minsv-R11}
\begin{aligned}
\sigma_{\min}(\bm{R}_{11})\ge \frac{\sigma_n(\bm{B})}{\sqrt{n(N-n)+1}}.
\end{aligned}
\end{eqnarray}
Making use of this property, we obtain an upper bound for the condition number of a principal submatrix of $\bm{\Phi}$.
\begin{theorem}\label{thm:exist-principlemat}
Let the eigenvalues of $\bm{\Phi}$ be $\lambda_1(\bm{\Phi})\ge\lambda_2(\bm{\Phi})\ge\cdots\ge\lambda_N(\bm{\Phi}) \geq 0$.
There exists an $(N-L)\times (N-L)$ principal submatrix $\bm{\Phi}_{22}$
of $\bm{\Phi}$ such that
\begin{eqnarray}\label{Hk>=Ak}
\begin{aligned}
{\kappa}(\bm{\Phi}_{22}) \leq [L(N-L)+1]\frac{\lambda_1(\bm{\Phi})}{\lambda_{N-L}(\bm{\Phi})}.
\end{aligned}
\end{eqnarray}
\end{theorem}
\begin{proof}
Since $\bm{\Phi}$ is positive semidefinite, there is an $N\times N$ $\bm{B}$ such that
$\bm{\Phi}=\bm{B}^{T}\bm{B}$.
Let $\bm{B}$ have an RRQR (\ref{eq:B=RRQR}) satisfying (\ref{ineq:minsv-R11})
with $n=N-L$.
Now notice
\begin{eqnarray}
\begin{aligned}
\bm{\Pi}^{T}\bm{\Phi\Pi}=&(\bm{B\Pi})^{T}(\bm{B\Pi})\\
=&\bm{R}^{T}\bm{R}\\
=&\left[\begin{array}{cc} \bm{R}_{11}^{T}\bm{R}_{11} & \bm{R}_{12}^{T}\bm{R}_{12} \\
                \bm{R}_{12}^{T}\bm{R}_{12} & \bm{R}_{12}^{T}\bm{R}_{12}+\bm{R}_{22}^{T}\bm{R}_{22}\end{array}\right]
\end{aligned}
\end{eqnarray}
to see that $\bm{\Phi}$ has an $(N-L)\times (N-L)$ principle submatrix $\bm{\Phi}_{22}=\bm{R}_{11}^{T}\bm{R}_{11}$ whose
smallest eigenvalue is
\begin{eqnarray}
\begin{aligned}
\sigma_{\min}(\bm{R}_{11})^2
   &\ge\left[\frac {\sigma_{N-L}(\bm{B})}{\sqrt{L(N-L)+1}}\right]^2\\
   &=\frac {\lambda_{N-L}(\bm{\Phi})}{L(N-L)+1},
\end{aligned}
\end{eqnarray}
because $\lambda_{N-L}(\bm{\Phi})=[\sigma_{N-L}(\bm{B})]^2$.
The result follows by noting that $\lambda_{\max}(\bm{\Phi}_{22}) \leq \lambda_1(\bm{\Phi})$.
\end{proof}

\subsection{Proof of Property 1}
The alignment matrix $\bm{\Phi}=[\phi_{ij}]$ in~(\ref{ML}, \ref{LS}, \ref{Spec}) is defined on the $K$-NN graph, which can be generated by various manifold learning methods, including LE~\cite{belkin2003laplacian}, LLE~\cite{roweis2000nonlinear} and LTSA~\cite{zhang2004principal}.

\begin{itemize}
\item \textbf{LE} uses a Laplacian graph matrix as the alignment matrix.
Suppose that $d_{ij}$ is the distance between sample $\bm{x}_i$ and its neighbor $\bm{x}_j$.
$\phi_{ij}=-d_{ij}$ if $\bm{x}_i$ and $\bm{x}_j$ are neighbors, otherwise $\phi_{ij}=0$, and $\phi_{ii}=\sum_j d_{ij}$.

\item \textbf{LLE} reconstructs the local geometry of the manifold by self-representation.
The entries of $\bm{\Phi}$ are given by
\begin{eqnarray*}
\begin{aligned}
\phi_{ij}=\delta_{ij}-w_{ij}-w_{ji}+\bm{w}_{i}{T}\bm{w}_j,~ i,j=1,..,N.
\end{aligned}
\end{eqnarray*}

\item \textbf{LTSA} approximates the local tangent space of each sample by Taylor expansion (\ref{taylor}).
The matrix $\bm{\Theta}_i=[\bm{\theta}_1^{i},...,\bm{\theta}_K^{i}]\in \mathbb{R}^{d\times K}$ are local latent variables of $\bm{X}_i$ in the tangent space.
Assuming that there exists an affine transformation between the global latent variables $\bm{Y}_i$ and the local ones $\bm{\Theta}_i$, we minimize the errors of the transformation by
\begin{eqnarray*}
\begin{aligned}
\min_{\bm{B}_i}~\|\bm{Y}_i\bm{B}_i\|_F^2,
\end{aligned}
\end{eqnarray*}
where $\bm{B}_i$ is the orthogonal projection whose null space is spanned by the columns of $[\bm{1},\bm{\Theta}_i]$.
Therefore, the alignment matrix $\bm{\Phi}=\sum_{i=1}^{N}\bm{S}_i\bm{B}_i\bm{B}_{i}^{T}\bm{S}_i^{T}$, where $\bm{S}_i$ are 0-1 selection matrix ensuring $\bm{X}_i=\bm{X}\bm{S}_i$.
\end{itemize}

Therefore, for the alignment matrix generated by these methods, only the elements corresponding to neighboring samples are nonzero.

\section*{Acknowledgment}
This work is supported in part by
NSF grants DMS-1317424, 
and CCF-1350616, 
and AFOSR grant FA9550-14-1-0342, as well as a gift from the Alfred P.\ Sloan Foundation.

\ifCLASSOPTIONcaptionsoff
  \newpage
\fi

\bibliographystyle{IEEEtran}
\bibliography{egbib}




\end{document}